\documentclass[lettersize,journal,twoside]{IEEEtran}

\usepackage{amsmath,amsfonts}
\usepackage{dsfont}
\usepackage{algorithmic}
\usepackage{algorithm}
\usepackage{array}
\usepackage[caption=false,font=normalsize,labelfont=sf,textfont=sf]{subfig}
\usepackage{textcomp}
\usepackage{stfloats}
\usepackage{url}
\usepackage{verbatim}
\usepackage{graphicx}
\usepackage{cite}

\usepackage[numbers]{natbib} % numbers option provides compact numerical references in the text. 
\usepackage{multicol}
\usepackage[bookmarks=true]{hyperref}
\usepackage{amssymb}
\usepackage{amsthm}
\usepackage{amsfonts}
\usepackage{mathtools}
\usepackage{bbm}
\usepackage{tikz}
\usepackage{dsfont}
\usepackage{optidef}
\usepackage{nccmath} 
\usepackage{xcolor}
\usepackage[inkscapelatex=false]{svg}
\usepackage[export]{adjustbox}
\usepackage{comment}
\usepackage[left=16.9mm,right=16.9mm,top=20.1mm,bottom=15.2mm]{geometry}
\newtheorem{theorem}{Theorem}[section]
\newtheorem{corollary}{Corollary}[theorem]
\newtheorem{lemma}[theorem]{Lemma}
\newtheorem{definition}{Definition}[section]
\newtheorem*{remark}{Remark}
\newcommand{\1}[1]{\mathbbm{1}\left[#1\right]}
\newcommand{\cur}[1]{\mathcal{#1}}

\begin{document}

\title{Learning Robot Safety from Sparse Human Feedback using Conformal Prediction}

\author{Aaron O. Feldman, Joseph A. Vincent, Maximilian Adang, JunEn Low, and Mac Schwager, \textit{Senior Member, IEEE}% <-this % stops a space
\thanks{This work was partly supported by ONR grant N00014-23-1-2354. Toyota Research Institute provided funds to support this work. The NASA University Leadership initiative (grant \#`80NSSC20M0163) provided funds to assist the authors with their research, but this article solely reflects the opinions and conclusions of its authors and not any NASA entity. The first author was also supported by NSF Graduate Research Fellowship grant 2146755, and the third author was supported on an NDSEG fellowship.}% <-this % stops a space
\thanks{The authors are with the Department of Aeronautics and Astronautics, Stanford University, Stanford, CA 94305, USA, {\texttt\footnotesize \{aofeldma, josephav, madang, jelow, schwager\}@stanford.edu}.}
% \thanks{The authors would like to acknowledge Keiko Nagami for help with the quadcopter dynamics/control and Tim Chen for help with Gaussian Splatting.}
}

\markboth{IEEE Transactions on Robotics}%
{Feldman, Vincent, Adang, Low, and Schwager: Learning Robot Safety}

\maketitle

\begin{abstract}
Ensuring robot safety can be challenging; user-defined constraints can miss edge cases, policies can become unsafe even when trained from safe data, and safety can be subjective. Thus, we learn about robot safety by showing policy trajectories to a human who flags unsafe behavior. From this binary feedback, we use the statistical method of conformal prediction to identify a region of states, potentially in learned latent space, guaranteed to contain a user-specified fraction of future policy errors. Our method is sample-efficient, as it builds on nearest neighbor classification and avoids withholding data as is common with conformal prediction. By alerting if the robot reaches the suspected unsafe region, we obtain a warning system that mimics the human's safety preferences with guaranteed miss rate. From video labeling, our system can detect when a quadcopter visuomotor policy will fail to steer through a designated gate. We present an approach for policy improvement by avoiding the suspected unsafe region. With it we improve a model predictive controller's safety, as shown in experimental testing with 30 quadcopter flights across 6 navigation tasks. Code and videos are provided.~\footnote{Code: \url{https://github.com/StanfordMSL/conformal-safety-learning}
\newline Project page: \url{https://stanfordmsl.github.io/conformal-safety-learning}
}
\end{abstract}


\section{Introduction}
\label{sec: intro}
In many robotic applications, it is challenging or infeasible to mathematically formalize constraints ensuring robot safety. This may be due to difficulty in modeling a complicated or uncertain environment. For instance, it is challenging to encode how a surgical robot should interact with deformable human tissue to remain minimally invasive. Ensuring safety may also be difficult due to the complexity of the robot policy itself. For example, it may be unclear what image observations cause a visuomotor manipulation policy to drop an object. Safety may also be subjective, hence difficult to specify. For example, users may have different tolerances for the speed a quadcopter may travel when navigating in their vicinity.

To tackle cases where robot safety cannot be easily defined or enforced, we propose to learn safety directly from human labels. Using trajectory data with instances where the policy became unsafe, we can infer a broader set of states, possibly in a learned latent space, where the human is likely to deem the policy unsafe in the future. Such data can either come from historical operation or be acquired in simulation (to make policy failure non-damaging). We rely on sparse human feedback to flag trajectories when they first become unsafe, making data labeling straightforward; the user can simply watch simulations and terminate when the robot acts unsafely. We could also use automatic labeling if a true safety definition is known but unusable at deployment (e.g., due to runtime or privileged state knowledge).

% Because safety is presumed challenging to specify, we rely on sparse human feedback to flag trajectories when they first become unsafe.

\begin{figure}
    \centering
    \includegraphics[width=\linewidth]{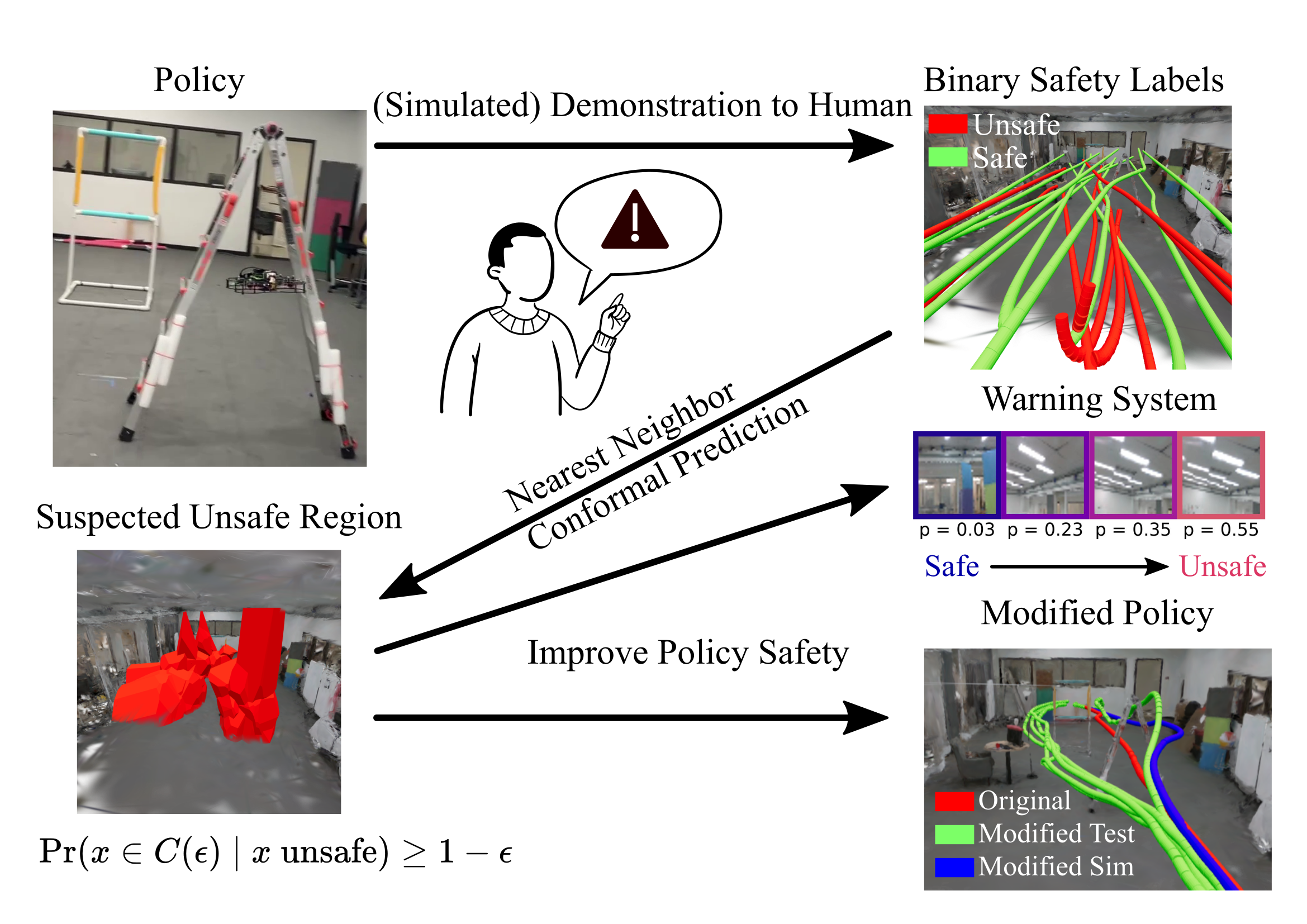}
    \caption{Our approach to learning robot safety from sparse human feedback. Given a robot policy, we repeatedly demonstrate it to a human, possibly in simulation, and have them terminate any trajectories which they deem unsafe. Using these binary labels, we apply conformal prediction to calibrate a nearest neighbor classifier and determine a suspected unsafe region $C(\epsilon)$ containing at least $1-\epsilon$ of states that would be deemed unsafe by the human. Using $C(\epsilon)$, we can improve the original policy's safety via a warning system or backup safety controller. The graphics show experiment results wherein we use human feedback to (i) develop a warning system for a visuomotor quadcopter policy and (ii) increase the safety of a quadcopter model predictive controller.}
    \label{fig: hero_fig}
\end{figure}

In any case, there may only be a few instances of policy failure, either due to a limited quantity of human-labeled data or because the policy is effective and failure is rare. To be sample-efficient, we parametrize the region of unsafe states using a variant of nearest neighbor classification (as opposed to more data-hungry machine learning methods). Furthermore, we present novel theory showing how we can, by reusing the original training data, apply the statistical procedure of conformal prediction to calibrate this classifier. This calibration provides a finite-sample generalization guarantee, ensuring that $1 - \epsilon$ of future unsafe states will get correctly flagged. It applies without distributional assumptions, even with few unsafe samples, and regardless of classifier quality. From a geometric perspective, the region suspected of being unsafe, termed the suspected unsafe sublevel (SUS) region $C(\epsilon)$, contains states falling below the calibrated classification threshold. Our approach is simple and interpretable: the user provides any dissimilarity measure (e.g., Euclidean distance) between states, and the resulting SUS region $C(\epsilon)$ often takes a geometrically convenient form (e.g., union of polyhedra) that can be visualized or used for avoidance constraints.

%In short, our approach applies conformal prediction to a variant of nearest neighbor classification to be both sample-efficient and provide guarantees about the fraction of unsafe states contained in the SUS region.

% Beyond sample efficiency and theoretical guarantees, our approach benefits from simplicity and interpretability. Besides the fraction $1-\epsilon$, the user simply provides a pairwise measure of dissimilarity between states (e.g., Euclidean distance), and the theoretical guarantees apply regardless of this chosen measure. The approach is interpretable as the SUS region $C(\epsilon)$ often adopts a geometrically convenient form, such as a union of polyhedra. Thus, it can be visualized in low dimensions or used to impose avoidance constraints on the original policy.

% Mac: I recommend that we claim the "two-sample" score function as our core method, and describe the one-sample as an ablation.

We discuss two approaches for using the obtained SUS region to improve policy safety: (i) as an auxiliary warning system, and (ii) for policy modification by avoiding the SUS region. The warning system alerts the user if the policy reaches the SUS region. Because of the conformal calibration it achieves a guaranteed miss rate, failing to alert in at most $\epsilon$ of unsafe policy executions. Beyond hard classification, the warning system can also act as a runtime safety monitor, providing an interpretable and calibrated scalar measure of safety throughout policy execution. The policy may also be modified to avoid entering the SUS region. As one such implementation, we consider modifying a model predictive control (MPC) policy to switch to a backup safety mode which steers to historical safe data when the MPC planned states are predicted to enter the SUS region. 

% This can be viewed as an extension to the auxiliary warning system which instead of aborting trajectories upon alert uses safe historical data to steer away from suspected unsafe states.

% We demonstrate our approach for learning a warning system and policy modification in simulated examples where a quadcopter must navigate while avoiding \textit{a priori} unknown obstacles. We also use our approach to learn a warning system from human video labeling of a visuomotor quadcopter policy, fitting the nearest neighbor classifier in a learned latent space of images. Using it, we can preemptively alert when the visuomotor policy will fail to pass through a designated gate. We also test our policy modification approach in hardware experiments where a quadcopter policy is modified to more cautiously navigate in a test environment using human feedback obtained in a Gaussian Splat simulator. We provide an overview of our approach and highlight some experimental results in Figure~\ref{fig: hero_fig}.

We demonstrate our approach by preemptively predicting failures of a visuomotor policy from image data using human video labeling, modifying a quadcopter policy to navigate while avoiding \textit{a priori} unknown obstacles, and in hardware experiments where we refine a quadcopter policy to more cautiously navigate in a test environment using human feedback. We provide an overview of our approach and highlight some experimental results in Figure~\ref{fig: hero_fig}.

% In summary, we develop a conformal prediction approach to learn safety from unsafe policy execution. Using this method, we can improve policy performance by augmenting the original policy with a warning system or by modifying the policy to avoid learned unsafe states. 

In summary, our primary contributions are,
\begin{itemize}
    \item We present novel theory to perform conformal prediction with nearest neighbor classification in closed-form without holding out data.
    \item From human feedback, we obtain a calibrated warning system guaranteed to flag a user-specified fraction of future policy errors.
    \item We can improve the robot policy by using a backup safety mode which is triggered upon a warning system alert.
\end{itemize}

% \item formulate a conformal prediction framework to learn a suspected unsafe region with guarantees regarding the fraction of errors it contains. 
% \item empirically demonstrate that the warning system outperforms other binary classifiers in the regime of limited error data.
% \item test in hardware the proposed method, using it to augment a visuomotor policy with a warning system that accepts high-dimensional image input.

% The paper is organized as follows. We cover related work on learning policy safety and conformal prediction in robotics in Section~\ref{sec: lit_rev}. In Section~\ref{sec: cp_overview} we briefly review conformal prediction, the underlying nonparametric statistical technique we use to build the suspected unsafe region. In Section~\ref{sec: problem_setting} we describe the problem setting and detail the key theoretical contributions in Section~\ref{sec: nncp}. We apply these in Section~\ref{sec: warning_sys} to develop a warning system and present the associated simulation experiments in Section~\ref{sec: warning_exp}. In Section~\ref{sec: policy_mod} we describe our approach for policy modification and in Section~\ref{sec: policy_mod_exp} present both simulation and hardware experiments demonstrating its success. We conclude and present avenues for future research in Section~\ref{sec: conclusion}.

\section{Literature Review}
\label{sec: lit_rev}

We first discuss related works in the robotics literature for data-driven learning of robot safety and then turn to related applications of conformal prediction in robotics.

\subsection{Learning Robot Safety}
Learning about robot safety from data has been tackled in a variety of ways. We categorize these approaches based on how they collect data: from expert demonstration, with binary stop feedback, or via expert assistance.

\subsubsection{Learning Constraints from Expert Demonstrations}
Several approaches have focused on learning safety constraints from an expert demonstrator that acts optimally while respecting the unknown safety constraints. Adopting a probabilistic view, some methods infer constraints that maximize demonstration likelihood \cite{scobee2020maximum, stocking2021discretizing, stocking2022, anwar2021inverse}. Others reason about implied safety from optimality conditions of the demonstrations, solving mixed integer feasibility programs \cite{chou2019learning, chou2019learningHighDim} or using the Karush-Kuhn-Tucker (KKT) conditions \cite{chou2020learning, chou2022gaussian}.

Our approach contrasts with methods learning from safe, optimal demonstration data in two key respects. Firstly, we do not assume access to a known reward function nor to demonstrated trajectories which are optimal while respecting the unknown constraints. Secondly, using conformal prediction, our recovered constraint set provides guaranteed coverage of the unsafe set without parametric/distributional assumptions.

\subsubsection{Learning Safety from Binary Feedback}

Several others have considered learning safety from a binary feedback function that only indicates when a state is unsafe. \cite{srinivasan2020learning, bharadhwaj2021conservative} repeatedly query this function during reinforcement learning training to simultaneously fit a safety-constrained policy and a safety critic which predicts future failure probability. \cite{Poletti2023} instead incorporates this learned function as a penalty in policy optimization. \cite{Thananjeyan2021} pretrains a safety critic from offline trajectory data and learns a recovery policy to minimize the critic's predicted risk, subsequently refining both through environment interaction. In contrast with these methods which iterate over (typically simulated) environment interaction, we only use offline batch data, with very few unsafe samples (e.g., $25$ not hundreds), and aim to make a given policy significantly safer in one iteration. Since we assume safety is unknown and rely on human labeling for binary feedback, sample efficiency is critical in our setting and precludes methods like \cite{thananjeyan2020safety} which query the binary safety function online.

% Instead of learning a safety critic, \cite{thananjeyan2020safety} queries the binary safety function online within a look-ahead MPC framework. Since we assume safety is either unknown or expensive to evaluate, we instead learn the SUS region and use the associated warning system as a proxy for binary safety feedback.

\subsubsection{Learning Safety from Human Interaction}
Other approaches learn safety from human interventions or corrections. The human may steer back to safety \cite{spencer2022expert}, provide directional guidance \cite{xie2024safe}, roughly abide constraints during shared control \cite{negar2016}, or physically outline safe boundaries \cite{Saveriano_2019}. While we too use human feedback for learning safety, we only require the human to terminate trajectories when they become unsafe, a simpler task than direct control/intervention.

% Several works have considered learning safety through interaction with an expert. \cite{spencer2022expert} assumes an expert intervenes to steer the robot back to safety and learns a safety score model from when the expert chooses to intervene and their actions. \cite{xie2024safe} uses human directional corrections to learn a parameterized but unknown safety constraint. In the context of shared human-robot control, \cite{negar2016} learns a motion assistance system to correct human inputs to better abide an unknown constraint. \cite{Saveriano_2019} has the human push the robot to outline the boundary of the safe set and incrementally learns a polyhedral approximation. 

\subsection{Conformal Prediction in Robotics}

% Note: could also cite AESOP, KnowNo, Haruki's CBF paper

Conformal prediction has gained popularity in robotics as a distribution-free method to improve robot safety. It has been used in two related applications: for anomaly detection and collision avoidance.

\subsubsection{Conformal Prediction for Anomaly Detection}

In the context of anomaly detection, which aims to flag instances unlike the training data, conformal prediction is used to calibrate the rate of false alarms without imposing distributional assumptions. Early works applied conformal prediction to nearest neighbor and density-based anomaly detectors \cite{laxhammar2010conformal, laxhammar2011, smith2014anomaly}. More recent efforts calibrate anomaly scores from learned models, such as trajectory forecasters \cite{contreras2024outofdistributionruntimeadaptationconformalized} or image autoencoders \cite{sinha2023closingloopruntimemonitors}. Instead of calibrating false alarm rate, \cite{luo2023sampleefficient} calibrates miss rate by applying conformal prediction to a predefined safety function using unsafe samples. They assume a predefined scalar safety function, while we directly learn safety from observed safe and unsafe states.

\subsubsection{Conformal Prediction for Collision Avoidance}

Conformal prediction has been used to improve robot collision avoidance by accounting for prediction and environment uncertainty. \cite{lindemann2023safe, strawn2023conformal,cleaveland2023conformal,stamouli2024recursively} use split conformal prediction to produce uncertainty sets for predicted motion of other agents which their robot policy then avoids. \cite{dixit2022adaptive, muthali2023multiagent} adaptively learn forecasting error within one trajectory, relying on time-series extensions to conformal prediction~\cite{gibbs2021, gibbs2022conformal}. In contrast with these works, we use full conformal prediction to avoid reserving calibration data and learn policy safety directly without pretrained models. By extracting the first unsafe state encountered in policy execution, we provide a guarantee on trajectory-level error probability without using time-series conformal methods that provide somewhat weaker results.

\section{Overview of Conformal Prediction}
\label{sec: cp_overview}
Conformal prediction \cite{shafer2007tutorial, angelopoulos2022gentle} starts with the question: Given $N+1$ random variables (scores) $s_1, ..., s_N, s_{N+1}$ which are exchangeable, equally likely under permutation, \footnote{Independent and identically distributed scores are also exchangeable.} what can be said about the value of $s_{N+1}$ relative to $s_1, ..., s_N$? If we assume no ties, the answer is that 
\begin{equation}
\label{eq: marginal}
\Pr(s_{N+1} \leq s_{(k)}) = \frac{k}{N+1}
\end{equation}
since $s_{N+1}$ is equally likely to achieve any rank in a sorting of $s_1, ..., s_{N+1}$. Here $s_{(k)}$ refers to the $k$'th order statistic of $s_1, ..., s_N$ (the $k$'th smallest value). If ties may occur, then the above equality is relaxed to $\geq$. % Notably, this result holds in expectation over all $s_1, s_2, ..., s_{N+1}$ \cite{angelopoulos2022gentle}.

Using this result, we can produce a probabilistic upper bound on $s_{N+1}$ using $s_1, ..., s_N$. Since $s_1, ..., s_N$ are random, the bound is random and so the user specifies with what probability the bound should hold: $1 - \epsilon$. We call $1-\epsilon$ the coverage probability and $\epsilon$ the miscoverage probability. Taking 
\begin{equation}
\label{eq: k_epsilon}
k(\epsilon) = \lceil(N+1)(1-\epsilon)\rceil,
\end{equation}
\begin{equation}
\label{eq: eps_marginal}
\Pr(s_{N+1} \leq s_{(k)}) \geq 1-\epsilon
\end{equation} where this holds for any distribution. Provided exchangeability holds, the scores may be obtained by applying any function to the original data. In this case, bounds on $s_{N+1}$ correspond to confidence sets for new data.

In split (or inductive) conformal prediction \cite{vovkinductiveCP}, the scores are computed by applying a fixed function $s_i = s(X_i, Y_i)$ to each IID datum $(X_i, Y_i)$. This approach is used in much of the recent robotics literature which requires held-out calibration data to compute $s_1, ..., s_N$. While Equation \ref{eq: eps_marginal} holds in expectation (over repeated draws of the calibration set), for one calibration set the coverage is known (assuming continuous IID scores) to follow a Beta distribution \cite{hulsman2022distributionfreefinitesampleguaranteessplit}. 

In contrast, full (or transductive) conformal prediction \cite{vovktransductiveCP, lei2018distribution}, used in this work, achieves exchangeability by swapping data ordering. Given IID points $x_1, ..., x_N$ and a candidate point $x$, the score function $s(D; x)$ measures the dissimilarity (as user-defined) between $x$ and the dataset $D$ of the remaining $N$ points (lower scores being more similar). \footnote{Split conformal prediction is a special case where the score depends only on the excluded point $x = (X, Y)$ so $s(D; x) = s(X, Y)$.} The scores are obtained by swapping $x$ with the different $x_i$ i.e., looking at each of the leave-one-out scores \footnote{Instead of replacing $x_i$ with $x$, the score may also keep both, using an augmented dataset of all $N+1$ points \cite{lei2018distribution}.}:
\begin{align}
\label{eq: swapping}
s_1^x = s(x, x_2, ..., x_N; x_1), s_2^x = s(x_1, x, x_3, ..., x_N; x_2), ..., \\ \nonumber
s_N^x = s(x_1, x_2, ..., x_{N-1}, x; x_N), s_{N+1}^x = s(x_1, x_2, .., x_N; x)
\end{align}
Assuming the score function is symmetric in the first $N$ arguments \cite{lei2018distribution} (i.e., $s(D; x)$ is invariant under reordering $D$), the resulting scores in Eq.~\ref{eq: swapping} will appear exchangeable when evaluating at $x = x_{N+1}$ exchangeable with $x_1,...,x_N$. Thus, using $k(\epsilon)$ as in Eq.~\ref{eq: k_epsilon}, the confidence set 
\begin{equation}
\label{eq: full conformal set}
C(\epsilon) = \{x \mid s_{N+1}^x \leq s_{(k)}^x\}
\end{equation}
satisfies, for $x_{N+1}$ exchangeable with $x_1,...,x_N$
\cite{lei2018distribution},
\begin{equation}
\label{eq: x_marginal}
    \Pr(x_{N+1} \in C(\epsilon)) \geq 1 - \epsilon.
\end{equation}

Full conformal prediction allows for more flexible score functions and does not require separate calibration data. However, the dependence on $x$ in both sides of Eq.~\ref{eq: full conformal set} often makes $C(\epsilon)$ intractable to compute in closed-form (although circumventions exist \cite{vovk2012crossconformal, ndiaye2019computing, ndiaye2022}). 
We avoid such complexity by using the nearest neighbor distance as the score function. In this case, we prove that we can pre-compute a single threshold $r$ defining $C(\epsilon)$. At inference, we need only check $s_{N+1}^x \leq r$ to check whether $x \in C(\epsilon)$. Furthermore, we can easily geometrically describe and visualize $C(\epsilon)$. 

% For instance, in machine learning, finding the scores would require many model re-fits, each time training but leaving out a different point. Approximate full conformal prediction approaches \cite{ndiaye2019computing, ndiaye2022} exist and cross-conformal prediction \cite{vovk2012crossconformal} has been developed to instead only re-fit on fewer, blocked subsets of the data. 

\section{Problem Setting}
\label{sec: problem_setting}

Let $\cur{X}$ denote the state space of the robot system and $\cur{U}$ refer to the associated action (control input) space. We assume a Markovian, time-invariant system evolving via, possibly stochastic, transition dynamics $\cur{T}: \cur{X} \times \cur{U} \rightarrow \cur{X}$ and with starting state distribution $\cur{D}_0$. Let $\cur{A} \subset \cur{X}$ be the true, unknown subset of states which are unsafe. We assume that $\cur{A}$ is time-invariant. 

We are given an original closed-loop policy $\pi : \cur{X} \rightarrow \cur{U}$ mapping from state to action. When executing $\pi$, we assume it is terminated when
\begin{itemize}
    \item Reaching a set $\cur{G} \subset \cur{X}$ of goal states.
    \item First reaching the unsafe set $\cur{A}$.
    \item Neither of the first two conditions have occurred within $T_{max}$ execution steps (a time-out condition).
\end{itemize}
We make no assumptions about the policy $\pi$, it may be feedback or model predictive control, obtained via imitation or reinforcement learning, and may be deterministic or stochastic.

Given policy $\pi$, we refer to a trajectory rollout $\tau$ of $\pi$ as a closed-loop execution of $\pi$. We obtain $\tau$ by initializing $x_0 \sim \cur{D}_0$, repeatedly execute action $u_t \sim \pi(x_t)$, and update the state $x_{t+1} \sim \cur{T}(x_t, u_t)$ until termination. We write $\tau = (x_0, x_1, ..., x_T)$. This stochastic procedure for obtaining $\tau$ defines an (implicit) associated distribution over resulting trajectories $\cur{D}_{\tau}$. In general, trajectories may differ in length so let $\tau(-1) = x_T$ denote the last state before termination. By definition $\tau(-1) \in \cur{A}$ if and only if $\tau$ is unsafe.

As the unsafe set $\cur{A}$ is unknown, we rely on a human labeler to declare trajectories unsafe. We collect safety data $\tau_i \sim \cur{D}_\tau$ by repeatedly demonstrating $\pi$ $P$ times to a human to obtain trajectory dataset $\{\tau_1, \tau_2, ..., \tau_P\}$. These rollouts may be collected in simulation or, when feasible, in the real world. To minimize distribution shift between data collection and policy deployment, the initial states $x_0 \sim \cur{D}_0$ should be sampled to match expected deployment conditions (e.g., quadcopter positions/velocities typical of uninterrupted autonomous flight).

In this demonstration phase, the observer is assumed to stop the trajectory whenever it reaches $\cur{A}$. Stopped trajectories are declared unsafe and otherwise declared safe. In terminating unsafe trajectories, the labeler can act preemptively and the learned SUS region will reflect their level of caution. For instance,  trajectories may be terminated when the robot's motion makes collision seem imminent. Preemptive labeling can also be achieved retroactively, by marking as unsafe the state recorded several timesteps before failure.

% Although $\pi$ may incorporate previous known/heuristic safety constraints (e.g., maintaining a minimum distance from obstacles), we must collect such data to reason about remaining policy safety failures. In other words, to characterize the unsafe regions reached by the policy inevitably requires reaching unsafe states. Therefore, we assume that during this demonstration phase the environment is made non-critical. This could be achieved by ensuring the human observer is prepared to intervene, by making the environment itself safer (e.g., padding surfaces), or by doing the entire demonstration phase in an accurate simulator. 

Let $D_{\cur{A}}$ store the $N$ unsafe trajectories collected and $D_{\overline{\cur{A}}}$ store the $P - N$ safe trajectories. Although we assume $N > 0$, our approach provides the same form of safety guarantees regardless of the size of $N$. \footnote{If we observed $N = 0$ and took a large $P$ then the original policy is presumably safe enough without needing modification or a warning system.} For unsafe trajectories, we also know the individual state, flagged by the labeler, where the trajectory first became unsafe. In each of the $N$ unsafe trajectories, we extract this final state, termed the error state: 
\begin{equation}
\label{eq: error_data}
    D = \{x_1 = \tau_1(-1), ..., x_N = \tau_N(-1)\} = \{\tau_i(-1) \ \forall \tau_i \in D_{\cur{A}}\}.
\end{equation}

Because we know individual error states $D$, we can reason about future policy safety at the level of individual states, instead of full higher-dimensional trajectories. Using $D$ we will learn a broader region $C(\epsilon)$ of suspected unsafe states where we anticipate future policy errors may occur.

Given policy trajectory distribution $\cur{D}_{\tau}$ and unknown unsafe set $\cur{A}$, there is an induced, unknown and possibly complex, distribution $F$ over the resulting error states $\tau(-1) \in \mathcal{A}$ reached by policy $\pi$ i.e., $F$ characterizes the distribution of terminal states $\tau(-1)$ conditioned on $\tau(-1) \in \cur{A}$. We refer to $F$ as the policy error distribution.~\footnote{$F$ has unknown probability density
\begin{equation*}
    \rho_F(x) = \frac{1}{\Pr[\tau(-1) \in \cur{A}]} \mathds{1}[x \in \cur{A}] \rho(x \mid \tau \sim \cur{D}_\tau, \tau(-1) = x)
\end{equation*}
where $\rho$ captures the density over terminal states (depending on $\cur{D}_{\tau}$).} By demonstrating $\pi$ to the human labeler, we have obtained $N$ IID samples $D$ from this unknown distribution $F$. Our objective is to use these observed error states to identify a set $C(\epsilon)$ such that a future policy error $\tau(-1) = x_{N+1} \sim F$ will likely be contained in $C(\epsilon)$:
\begin{equation}
\label{eq: problem_def}
    \Pr[x_{N+1} \in C(\epsilon)] = \Pr[\tau(-1) \in C(\epsilon) \mid \tau(-1) \in \cur{A}] \geq 1 - \epsilon
\end{equation}
i.e., $C(\epsilon)$ should cover/contain $\geq 1-\epsilon$ of future error states.

\section{Nearest Neighbor Conformal Prediction}
\label{sec: nncp}

\subsection{Main Results}
In this section, we describe how we use full conformal prediction applied to the collected error states $D$ to identify the SUS region $C(\epsilon)$ that in expectation contains at least a $1-\epsilon$ fraction of error states that would be reached in future executions of $\pi$ i.e., satisfying Eq.~\ref{eq: problem_def}. $C(\epsilon)$ contains states whose nearest neighbor among the error states $D$ falls below a distance threshold $r$ i.e., in the sublevel set of the nearest neighbor classifier fit using $D$. We thus term $C(\epsilon)$ the suspected unsafe sublevel (SUS) region. To guarantee that $C(\epsilon)$ covers $1-\epsilon$ of future policy error states, satisfying Eq.~\ref{eq: problem_def}, we use full conformal prediction to select/calibrate the threshold $r$. Notably, we do so by re-using the points in $D$, and do not require a separate calibration dataset (as in split conformal prediction), which is especially useful when we have just a few unsafe data points. 

To use full conformal prediction, we define a nearest neighbor score function:

\begin{equation}
\label{eq: score func}
    s(D; x) := \min_{x' \in D} d(x', x)
\end{equation}
where $d(x', x)$ is any measure of dissimilarity between two states $x$ and $x'$ (smaller values indicate more similarity). Thus, the score function is low when a new candidate state $x$ (observed during deployment) is similar to the previously observed error states in $D$. Regardless of how similarity is defined, through  $d$, our analysis guarantees $C(\epsilon)$ will achieve at least $1-\epsilon$ coverage (Eq.~\ref{eq: problem_def}). While $d$ can be non-metric or asymmetric, certain choices can more efficiently discriminate safety, yielding a smaller SUS region and fewer false alarms.

Typically in full conformal prediction, $C(\epsilon)$ is described implicitly and requires knowing the query point $x$ to check inclusion \cite{angelopoulos2024theoretical, lei2018distribution}. Using all $N+1$ points, we would swap $x$ with each $x_i$ in $D$ (Eq.~\ref{eq: swapping}), computing Eq.~\ref{eq: score func} $s_i^x = s(D_{-i}; x_i)$ between $x_i$ and the other $N$ points $D_{-i} = \{x_1, x_2, ..., x_{i-1}, x,  x_{i+1}, x_N\}$, including $x$. We would also compute $s_{N+1}^x = s(D; x)$ and conclude $x \in C(\epsilon)$ when $s_{N+1}^x \leq s_{(k)}^x$ (Eq.~\ref{eq: full conformal set}).

In Theorem~\ref{thm: main_cp} we show that, in particular for the nearest neighbor score, we can modify the standard procedure to only compare $s_{N+1}^x$ against intra-dataset scores $\alpha_i = \min_{x' \in D, x' \neq x_i} d(x', x_i)$ i.e., we can replace $s_i^x$ by $\alpha_i$ and exclude $x$. Theorem~\ref{thm: main_cp} enables us to perform full conformal prediction in closed-form offline (without knowing $x$ and swapping), yielding an explicit $C(\epsilon)$ defined by precomputed threshold $r = \alpha_{(k)}$. We defer the proof for this and subsequent theorems to the Appendix.

\begin{theorem}[Closed-Form Conformal Prediction]
    \label{thm: main_cp}
    Let $D = \{x_1, ..., x_N\}$ be states drawn IID from any, possibly unknown, distribution $F$ and suppose we are given a miscoverage rate $\epsilon$. Let $\alpha_1, ..., \alpha_N$ be the intra-data nearest neighbor values
        \begin{equation}
            \alpha_i = \min_{x' \in D, x' \neq x_i} d(x', x_i)
        \end{equation}
    and let $k = k(\epsilon) \leq N$ (Eq.~\ref{eq: k_epsilon}).
    Using $r = \alpha_{(k)}$ to form 
    \begin{equation}
         C(\epsilon) = \{x \mid s(D; x) \leq r\}
    \end{equation} 
    satisfies for new states $x_{N+1} \sim F$
    \begin{equation}
        \Pr[x_{N+1} \in C(\epsilon)] \geq 1 - \epsilon.
    \end{equation}
\end{theorem}

In our case, the previously observed error states in $D$ are IID samples from the policy error distribution $F$, so applying Theorem~\ref{thm: main_cp} guarantees that the resulting SUS region $C(\epsilon)$ constructed from $D$ satisfies Eq.~\ref{eq: problem_def}.

Theorem~\ref{thm: main_cp} applies regardless of $N$, but requires $k(\epsilon) \leq N$ i.e., $\epsilon \geq 1/(N+1)$. Choice of $N$ can also impact the conservatism of $C(\epsilon)$. With too few points nearest neighbor classification can fail to effectively distinguish safe data (i.e., $s(D; x)$ could be similar for both new safe and unsafe points). 

% In Theorem~\ref{thm: main_cp}, we showed that we could offline find a cutoff $r$ such that $C(\epsilon) = \{x \mid s(D; x) \leq r\}$. 
Because $s(D; x) = \min_{x' \in D} d(x', x)$, we can equivalently describe $C(\epsilon)$ as a union of sublevel sets of $d$ about each datum $x_i$. We formalize this result in Theorem~\ref{thm: union_cp}. 

\begin{theorem}[Geometric Coverage Set]
    \label{thm: union_cp}
    We may equivalently describe $C(\epsilon)$ from Theorem~\ref{thm: main_cp} (using the cutoff value $r$) as
    \begin{subequations}
        \begin{gather}
            C(\epsilon) = \cup_{i=1}^N C_i \\
            C_i = \{x \mid d(x_i, x) \leq r\}.
        \end{gather}
    \end{subequations}
\end{theorem}

With this representation, we can geometrically describe $C(\epsilon)$ for many standard choices of $d$. With Euclidean distance, $C(\epsilon)$ is a union of closed balls $C_i$, of shared radius $r$, each centered at one of the states $x_i$. Figure~\ref{fig: geom_fig} (top) visualizes the $C(\epsilon)$ geometry with Euclidean distance.

\begin{figure}
     \centering
     \includegraphics[width=0.9\linewidth]{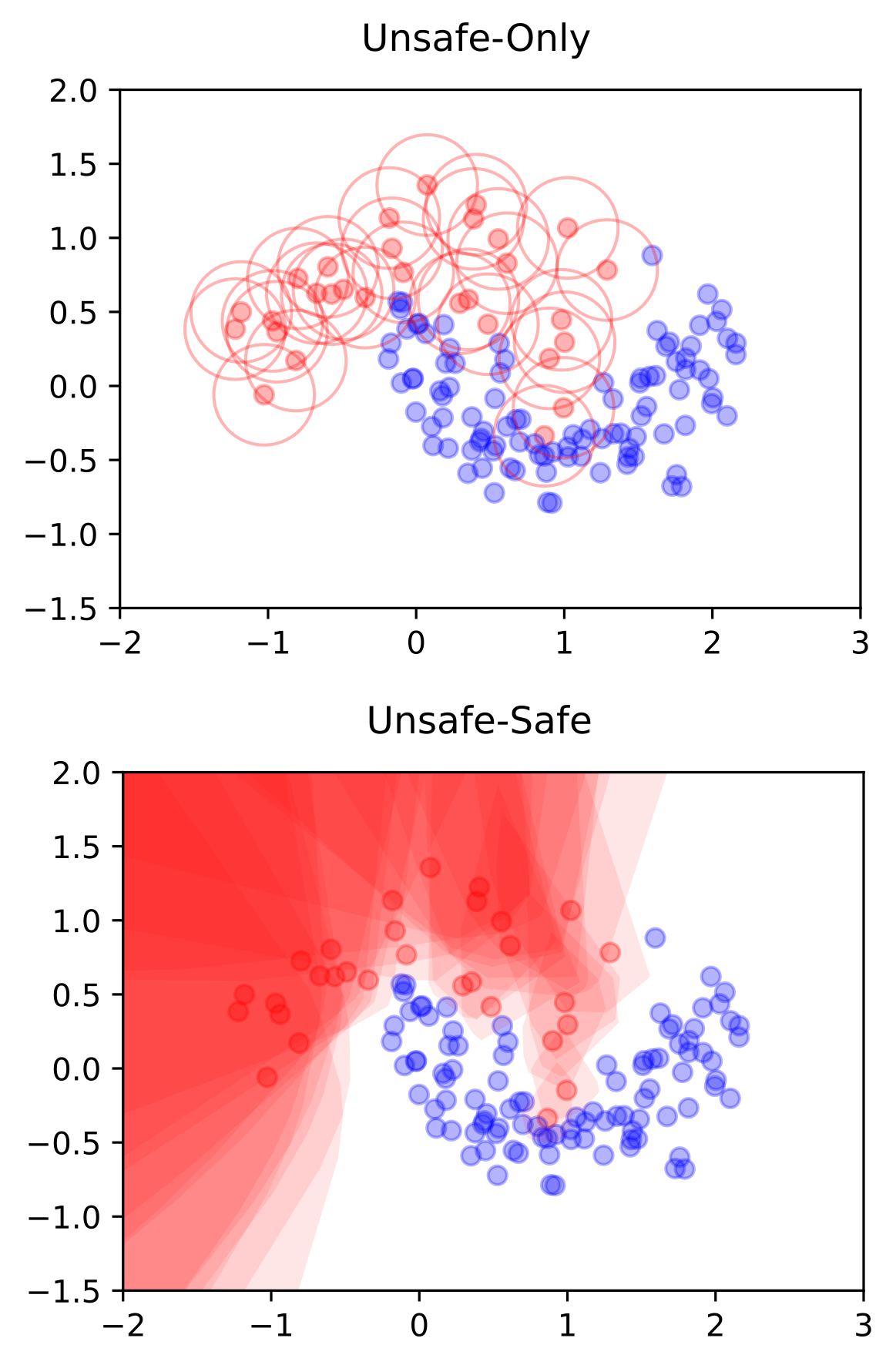}
    \caption{Visualizing the geometry of the unsafe-only (Eq.~\ref{eq: score func}) and unsafe-safe (Eq.~\ref{eq: asym_score_func}) conformal covering set $C(\epsilon)$. We sample $N = 30$ unsafe points (shown in red) and $M = 100$ safe points (shown in blue) and request miscoverage of $\epsilon = 0.1$. In the top subfigure, we plot $C(\epsilon)$ for the unsafe-only case using Euclidean distance. In the bottom subfigure, we use the difference of squared Euclidean distance for the unsafe-safe case (Eq.~\ref{eq: two_sample_special}) resulting in a union of polyhedra (see subsection \ref{subsec: two_sample}).}
    \label{fig: geom_fig}
\end{figure}

While Theorem \ref{thm: main_cp} ensures $C(\epsilon)$ achieves at least $1-\epsilon$ coverage, we can provide a theoretical upper bound on the coverage, distinguishing between the case of a symmetric versus asymmetric choice of $d$. The upper bound holds when the probability of having distance ties is $0$. Informally, this holds in the general setting of a continuous state space and for non-trivial $d$ (e.g., avoiding $d(x,y) = 0$).

\begin{theorem}[Overcoverage Bound: Symmetric Case]
    \label{thm: overcoverage_sym}
    Consider a symmetric pairwise function $d$ which satisfies for IID samples $x_1, x_2, ..., x_N, x_{N+1} \sim F$, the associated pairwise distances $\{d(x_i, x_j)\}_{i < j}$ feature no repeated values with probability 1.
    Then, $C(\epsilon)$ as constructed in Theorem \ref{thm: main_cp} satisfies:
    \begin{equation}
        \Pr[x_{N+1} \in C(\epsilon)] \leq 1 - \epsilon + 2/(N+1).
    \end{equation}
\end{theorem}

In establishing Theorem~\ref{thm: overcoverage_sym}, we build on similar logic (see Lemma~\ref{lemma: cp_ties}) as existing upper bounds for conformal prediction \cite{lei2018distribution, angelopoulos2022gentle}. However, the proof requires special care regarding the possibility of score ties. With symmetric $d$, ties occur when two points are mutual nearest neighbors, yielding overcoverage of $2/(N+1)$, instead of the typical $1/(N+1)$ without ties \cite{lei2018distribution}. As in Theorem~\ref{thm: main_cp}, we must reason about the effect of replacing the scores $s_i^x$, which depend on query point $x$, with their intra-dataset counterparts $\alpha_i$ (which exclude $x$). With symmetric $d$, this incurs no additional overcoverage. For asymmetric $d$ (Theorem~\ref{thm: overcoverage_asym}), this increases the bound by $1/N$.

Theorem~\ref{thm: main_cp} ensures that, in expectation, $C(\epsilon)$ covers at least $1-\epsilon$ of future unsafe states, regardless of $N$, while Theorem~\ref{thm: overcoverage_sym} guarantees a coverage at most $1-\epsilon+2/(N+1)$. Together, these results establish tight finite-sample generalization bounds on the SUS region’s test classification of unsafe states, in expectation over the collected training data $D$. However, the realized coverage i.e., $\Pr[x_{N+1} \in C(\epsilon)]$ for $x_{N+1} \sim F$ fluctuates across repeated draws of the data $D$. In Figure~\ref{fig: hist_cov} (top) we plot a histogram of $C(\epsilon)$ coverage across repeated $D$ draws. The average coverage indeed lies between $[1-\epsilon, 1-\epsilon + 2/(N+1)]$, supporting our theoretical analysis. As an assessment of the learned SUS region's stability, we observe that the realized coverage across repetitions generally remains close to $1-\epsilon$ (i.e., without fluctuating wildly). We repeated the procedure with larger $N = 60, 120$ and overlay silhouettes of the resulting histograms. With increasing $N$, the histogram shifts towards $1 - \epsilon$. Although not directly applicable, related analysis for split conformal prediction with IID scores~\cite{hulsman2022distributionfreefinitesampleguaranteessplit, angelopoulos2024theoretical} has shown that coverage follows a Beta distribution which asymptotically concentrates around $1-\epsilon$.
% We illustrate the validity of Theorem~\ref{thm: main_cp} and Theorem~\ref{thm: overcoverage_sym} in Figure~\ref{fig: hist_cov} (top) which shows that in repeated generations of $C(\epsilon)$ the coverage fluctuates but indeed has the average coverage between $1-\epsilon$ and $1-\epsilon+2/(N+1)$. Besides reducing the $2/(N+1)$ theoretical overcoverage gap, increasing $N$ should also decrease the coverage fluctuation across repetitions.~\footnote{This is theoretically justified in the IID case via the Beta distribution \cite{hulsman2022distributionfreefinitesampleguaranteessplit}.}

Theorem~\ref{thm: overcoverage_sym} assumed symmetric $d$. Theorem~\ref{thm: overcoverage_asym} provides a similar overcoverage bound for asymmetric $d$. This result applies in the unsafe-safe nearest neighbor approach (subsection~\ref{subsec: two_sample}) and Figure~\ref{fig: hist_cov} (bottom) illustrates the bound's validity.

\begin{theorem}[Overcoverage Bound: Asymmetric Case]
    \label{thm: overcoverage_asym}
    Consider an asymmetric pairwise function $d$ such that for IID samples $x_1, x_2, ..., x_N, x_{N+1} \sim F$, the associated pairwise distances $\{d(x_i, x_j)\}_{i \neq j}$ feature no repeated values with probability 1.
    Then, $C(\epsilon)$ as constructed in Theorem~\ref{thm: main_cp} satisfies:
    \begin{equation}
        \Pr[x_{N+1} \in C(\epsilon)] \leq 1 - \epsilon + 1/(N+1) + 1/N.
    \end{equation}
\end{theorem}

% \begin{figure}
%      \centering
%      \subfloat[]{\includesvg[width=0.9\linewidth]{Figures/theory_figs/one_cov.svg}\label{fig: one_cov}}
%      \hfill
%      \subfloat[]{\includesvg[width=0.9\linewidth]{Figures/theory_figs/two_cov.svg}\label{fig: two_cov}}
%     \caption{Empirical verification of the theoretical coverage guarantees. For both the unsafe-only and unsafe-safe cases, we construct $C(\epsilon)$ $1000$ times with fresh data (using the same setup of Figure~\ref{fig: geom_fig}) and evaluate coverage using $1000$ new unsafe test points. We plot a histogram of the coverage over these repetitions and plot the the average coverage as a green dashed line. The average coverage lies between the theoretical lower (Theorem~\ref{thm: main_cp}) and upper bounds (Theorem~\ref{thm: overcoverage_sym} for the unsafe-only case and Theorem~\ref{thm: overcoverage_asym} for the unsafe-safe case), shown respectively as red and blue dashed lines.}
%     \label{fig: hist_cov}
% \end{figure}

\begin{figure}
     \centering
     \includegraphics[width=0.9\linewidth]{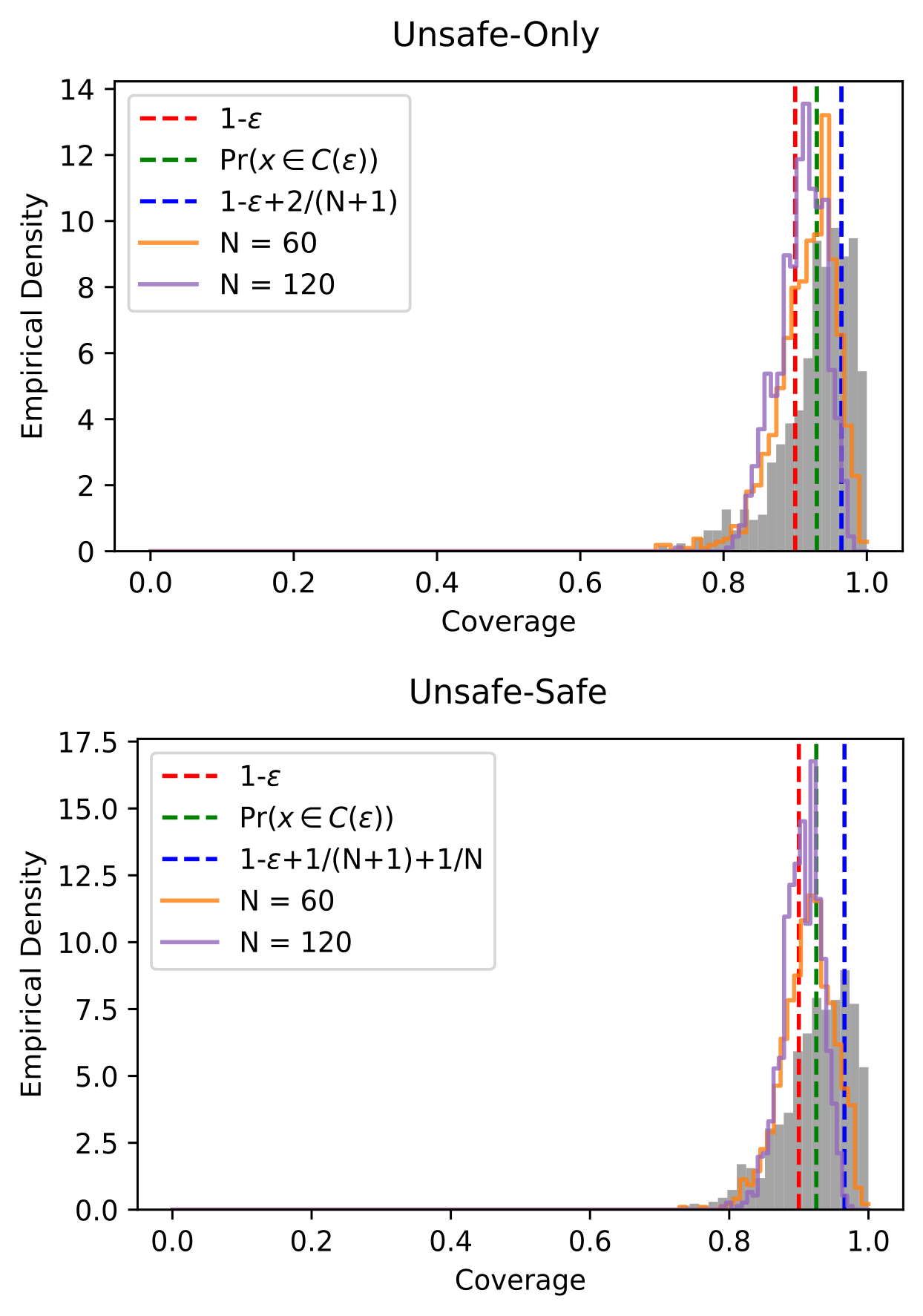}
    \caption{Empirical verification of the theoretical coverage guarantees. For both unsafe-only and unsafe-safe cases, we construct $C(\epsilon)$ $1000$ times with fresh data (using the $N = 30$ setup of Figure~\ref{fig: geom_fig}) and evaluate coverage, i.e., the probability a new unsafe state is contained in $C(\epsilon)$, using $1000$ unsafe test points. We plot a histogram (in gray) of the coverage over these repetitions and plot the the average coverage as a green dashed line. The average coverage lies between the theoretical lower (Theorem~\ref{thm: main_cp}) and upper bounds (Theorem~\ref{thm: overcoverage_sym} for the unsafe-only case and Theorem~\ref{thm: overcoverage_asym} for the unsafe-safe case), shown respectively as red and blue dashed lines. We repeat the experiment for $N = 60, 120$ and show silhouettes of the resulting histograms (in orange/purple). Increasing $N$ shifts the histograms towards $1 - \epsilon$.}
    \label{fig: hist_cov}
\end{figure}

\subsection{Unsafe-Safe Nearest Neighbor Extension}
\label{subsec: two_sample}

Until now we constructed $C(\epsilon)$ only using the error states from unsafe trajectories $D_F = \{x_1, ..., x_N\}$ (previously referred to as $D$), sampled from the unknown distribution $F$ over policy errors. For $C(\epsilon)$ to better distinguish safe from unsafe data, we can also use data from safe trajectories. Let $D_G = \{y_1, ..., y_M\}$ be points generated under the safe distribution $G$ i.e., states not in $\cur{A}$ sampled from within the stochastically-generated safe trajectories induced under $\cur{D}_{\tau}$. By using this safe data, $C(\epsilon)$ will still contain $1-\epsilon$ of future unsafe points drawn from $F$ but hopefully fewer safe points drawn from $G$. To incorporate the safe data, we difference the nearest neighbor distances for the unsafe and safe data,
\begin{equation}
\label{eq: asym_score_func}
s(D_F; x) = \min_{x' \in D_F} d(x', x) - \min_{y' \in D_G} d(y', x).
\end{equation}
This margin-style score function provides an intuitive approach to discriminating the unsafe from safe data and can perform surprisingly well (see Section~\ref{sec: warning_exp}) in the setting we face of limited and imbalanced data. It can also be motivated via Gaussian kernel density estimation and likelihood ratio testing (see the Appendix).

However, Eq.~\ref{eq: asym_score_func} seemingly no longer follows the nearest neighbor form (Eq.~\ref{eq: score func}) key to obtaining the SUS region $C(\epsilon)$ in closed form (Theorem~\ref{thm: main_cp}), via a modified full conformal prediction procedure. In fact, the unsafe-safe score can be reduced to nearest neighbor form with an appropriately modified dissimilarity measure. To do so, we exploit that $d$ may be asymmetric in the nearest neighbor score: 

\begin{subequations}
\begin{align}
    s(D_F; x) = \min_{x' \in D_F} d(x', x) - \min_{y' \in D_G} d(y', x) = \\
        \min_{x' \in D_F} [d(x', x) - \min_{y' \in D_G} d(y', x)] = \\
        \min_{x' \in D_F} [\Tilde{d}(x', x)].
\end{align}
\end{subequations}
We have thus recast this as a nearest neighbor score with an asymmetric measure $\Tilde{d}$:
\begin{equation}
    \label{eq: asym_dist}
    \Tilde{d}(x', x) = d(x', x) - \min_{y' \in D_G} d(y', x).
\end{equation}
We may now again rely on Theorem~\ref{thm: main_cp} to obtain $C(\epsilon)$ in closed-form, guaranteeing miscoverage at most $\epsilon$ (Eq.~\ref{eq: problem_def}) by calibrating using the samples from $D_F$, treating the unsafe-safe score Eq.~\ref{eq: asym_score_func} as simply using a more complicated dissimilarity measure $\Tilde{d}$ in Eq.~\ref{eq: score func}.
For the unsafe-safe score, the SUS region $C(\epsilon)$ cutoff $r$ is now obtained by computing 
\begin{subequations}
\begin{gather}
\alpha_i = \min_{x' \in D_F, x' \neq x_i} \Tilde{d}(x', x_i) = \\ \label{two_sided_alpha}
\min_{x' \in D_F, x' \neq x_i} d(x', x_i) - \min_{y' \in D_G} d(y', x_i)
\end{gather}
\end{subequations}
for $i=1,2,...,N$ and setting $r = \alpha_{(k)}$.

A notable special case is squared Euclidean distance: 
\begin{equation}
\label{eq: two_sample_special}
s(D_F; x) = \min_{x' \in D_F} ||x' - x||^2 - \min_{y' \in D_G} ||y' - y||^2.
\end{equation} In this case, the resulting conformal set $s(D_F; x) \leq r$ again has convenient geometry. Applying Theorem~\ref{thm: union_cp} to the asymmetric distance Eq.~\ref{eq: asym_dist}, we obtain $C(\epsilon) = \cup_{i=1}^N C_i$ where 
\begin{subequations}
\begin{gather}
    C_i = \{x \mid \Tilde{d}(x_i, x) \leq r\} = \\
    \{x \mid d(x_i, x) - \min_{y' \in D_G} d(y', x) \leq r\} = \\
    \{x \mid d(x_i, x) - d(y_j, x) \leq r \ \forall j=1:M\}.
\end{gather}
\end{subequations}
With squared Euclidean distance, the constraints
\begin{equation}
    ||x_i - x||^2 - ||y_j - x||^2 \leq r \ \forall j=1:M
\end{equation}
reduce to half-spaces $a_{ij}^T x \leq b_{ij}$ with $a_{ij} = y_j - x_i$, $b_{ij} = \frac{r + ||y_j||^2 - ||x_i||^2}{2}$. Thus, each
$C_i$ is a polyhedron and $C(\epsilon)$ is a union of $N$ polyhedra each with $M$ constraints. 

% \footnote{If we had used for the score the ratio of nearest neighbor distances instead of the difference, we would not get such a convenient geometric description.}

Fig.~\ref{fig: geom_fig} (bottom) shows the geometry of $C(\epsilon)$ constructed using the unsafe-safe score (Eq.~\ref{eq: asym_score_func}). Qualitatively $C(\epsilon)$ constructed with the unsafe-safe score intrudes less into the region of safe points. In our later warning system experiments (Section~\ref{sec: warning_exp}) we show this results in fewer false alarms.

\subsection{Probabilistic Interpretation}
\label{subsec: p_val}
% We decide a new state is unsafe if $S(D; x) = \min_{x'} d(x', x) \leq r$ which may be viewed as a nearest neighbor classifier which declares a new point $x$ like $x_1, ..., x_N$ if $x$ is sufficiently close to one of $x_1, ..., x_N$. From the classification perspective, we choose $r$ to guarantee the classifier miss rate i.e., when $x$ belongs to the same class as $x_1, ..., x_N$, we only miss declaring it as such with probability at most $\epsilon$: $\Pr(s(D; x) > r \mid x \sim F) \leq \epsilon$.

Beyond hard classification of whether a state is considered unsafe, we can use the nearest neighbor conformal approach to provide a probabilistic interpretation of safety, using the conformal prediction $p$-value \cite{shafer2007tutorial, angelopoulos2024theoretical}.

\begin{definition}[Conformal $p$-Value]
For state $x$, the $p$-value is the smallest $\epsilon$ for which $x \notin C(\epsilon)$, denoted $\epsilon^*$. 
\end{definition}

This matches the standard definition of $p$-value as the smallest size $\epsilon$ under which a hypothesis test rejects. In this setting, we default to the null hypothesis that $x \sim F$ (i.e., $x$ is like $\{x_i\}_{i=1}^N$) and reject when $x \notin C(\epsilon)$. Intuitively, smaller $\epsilon^*$ reflects stronger evidence that $x_{N+1}$ is unlike $\{x_i\}_{i=1}^N$. Hence, for $\{x_i\}_{i=1}^N$ unsafe a larger $\epsilon^* \in [0,1]$ indicates $x_{N+1}$ is seemingly less safe. For a given cutoff $\epsilon$ and state $x_{N+1}$ with $p$-value $\epsilon^*$, by definition $x_{N+1} \in C(\epsilon)$ if and only if $\epsilon^* \geq \epsilon$. Thus, the $p$-values can be viewed as describing the increasing growth of the suspected unsafe region $C(\epsilon)$ as $\epsilon$ is decreased. Figure~\ref{fig: p_fig} visualizes how the $p$-value varies throughout space when using the unsafe-only and unsafe-safe nearest neighbor approaches. Qualitatively, lower $p$-values are attained in the safe region using the unsafe-safe approach, suggesting it can better distinguish safe from unsafe data.

% \begin{figure}
%      \centering
%      \subfloat[]{\includesvg[width=0.9\linewidth]{Figures/theory_figs/one_p.svg}}\label{fig: one_p}
%      \hfill
%      \subfloat[]{\includesvg[width=0.9\linewidth]{Figures/theory_figs/two_p.svg}}\label{fig: two_p}
%     \caption{Visualizing the $p$-value associated with $C(\epsilon)$. Using the same setup as in Fig.~\ref{fig: geom_fig}, we vary $\epsilon$ to visualize the $p$-value level sets for the unsafe-only (Fig.~\ref{fig: one_p}) and unsafe-safe (Fig.~\ref{fig: two_p}) cases. While both the unsafe-only and unsafe-safe cases provide valid coverage of at least $1-\epsilon$, based on the $p$-value visualizations we qualitatively observe that the unsafe-safe approach better distinguishes the unsafe and safe samples.}
%     \label{fig: p_fig}
% \end{figure}

\begin{figure}
     \centering
     \includegraphics[width=0.9\linewidth]{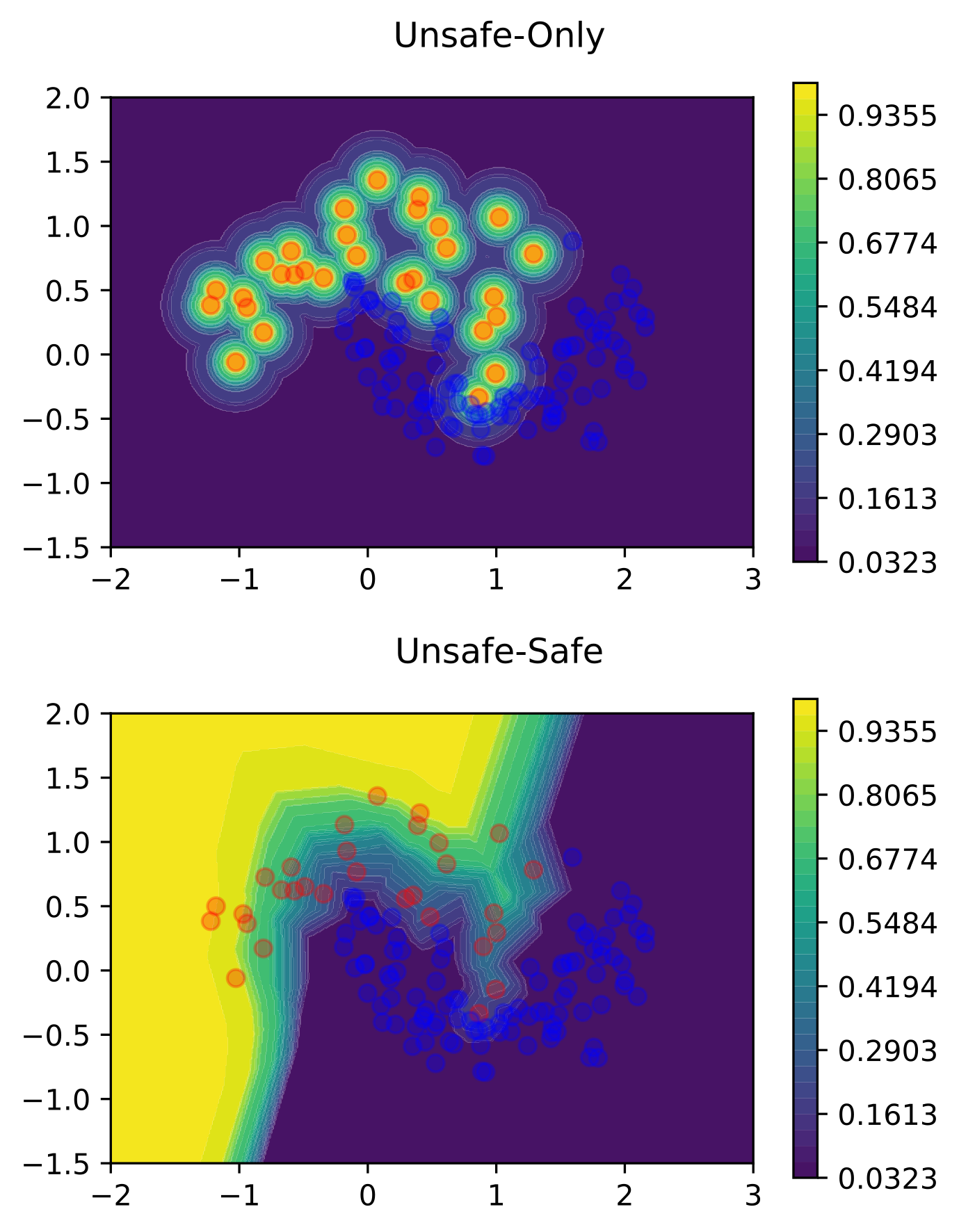}
    \caption{Visualizing the $p$-value associated with $C(\epsilon)$. Using the same setup as in Fig.~\ref{fig: geom_fig}, we vary $\epsilon$ (from $1/(N+1)$ to $1$ in increments of $1/(N+1)$, where $N = 30$) to visualize the $p$-value level sets for the unsafe-only (top) and unsafe-safe (bottom) cases. While both cases provide valid coverage of at least $1-\epsilon$, based on the $p$-value visualizations we qualitatively observe that the unsafe-safe approach better distinguishes the unsafe and safe samples.}
    \label{fig: p_fig}
\end{figure}

\subsection{Combining with Representation Learning}
\label{subsec: representations}

The conformal guarantee of $1-\epsilon$ coverage holds for any $d$, but better choices can yield fewer false alarms i.e., declaring a state $x$ unsafe when it is actually safe. We can therefore learn more effective $d$ by transforming the original data $\phi(x)$ to better distinguish safe and unsafe data. The conformal nearest neighbor score function post-transformation is
\begin{equation}
\label{eq: transformed_score}
    s(D; x) = \min_{x'} d(\phi(x'), \phi(x)).
\end{equation}
which may be viewed as simply an alternative choice of $d$. As long as $\phi$ is trained without using the calibration data i.e., the error states $D$, coverage remains valid.

Effective strategies for learning representation $\phi(x)$ include classical methods: principle component analysis (PCA), its kernelized variant~\cite{turk1991face, kernelface}, and large margin nearest neighbor (LMNN)~\cite{lmnn} and modern approaches: DINO/DINOv2~\cite{dino, dinov2} and CLIP~\cite{clip}. Even in the big-data regime, these recent methods have shown good performance for images by combining representation learning and nearest neighbor classification. Representation learning has also been combined effectively with conformal prediction; \cite{ghosh2023improving} used learned representations of neural network classifiers in a localized conformal procedure \cite{guan2023localized} to produce provably tighter confidence sets. In Section~\ref{sec: warning_exp}, we use representation learning with our nearest neighbor conformal approach to scale to image data and detect unsafe behavior for a visuomotor quadcopter policy.

% The representation learning problem of finding $\phi$ is well-studied. Methods can be unsupervised or supervised, using class labels to better separate class distributions. Classical approaches include principle component analysis (PCA) and its kernelized variant \cite{turk1991face, kernelface}, unsupervised, and large margin nearest neighbor (LMNN) \cite{lmnn}, supervised. DINO/DINOv2 \cite{dino, dinov2}, unsupervised, and CLIP \cite{clip}, supervised, are more recent notable examples. Even in the big-data regime, these methods have shown good performance for images by combining representation learning and nearest neighbor classification. Representation learning has also been combined effectively with conformal prediction; \cite{ghosh2023improving} used learned representations of neural network classifiers in a localized conformal procedure \cite{guan2023localized} to produce provably tighter confidence sets. In Section~\ref{sec: warning_exp}, we use representation learning with our nearest neighbor conformal approach to scale to image data and detect unsafe behavior for a visuomotor quadcopter policy.

% and \cite{kumar2022reliable} uses it with conformal prediction.

% different distorted images are generated (e.g., via rotation or cropping) and their representations are encouraged to be similar. The supervised analogue, termed contrastive learning, fits a transformation to maximize similarity within-class but minimize it between-class. In this category, CLIP is a notable example \cite{clip} and \cite{kumar2022reliable} uses it with conformal prediction. 

\section{Warning System}
\label{sec: warning_sys}

Using Section~\ref{sec: nncp}, we augment the original policy $\pi$ with a warning system whose miss rate, the probability of failing to alert when the human would, is at most a user-specified $\epsilon$.

% Directly applying the results of Section~\ref{sec: nncp}, we can augment the original $\pi$, under which human feedback was collected, with a warning system achieving a guaranteed miss rate i.e., the probability of failing to alert when the human would have is no more than the user-specified $\epsilon$.

Offline, using $N$ error states $D$ obtained when demonstrating $\pi$, we compute the intra-dataset nearest neighbor distances $\{\alpha_i\}_{i=1}^N$ (Eq.~\ref{two_sided_alpha}), $k(\epsilon) = \lceil(1-\epsilon)(N+1)\rceil$, and $r = \alpha_{(k)}$. During policy execution, for each state $x$, we compute the nearest neighbor score $s(D; x)$ and alert if $s(D; x) \leq r$ i.e., $x \in C(\epsilon)$. Alerts may be used to abort policy execution (e.g., transition a quadcopter to hover), signal for human takeover, or trigger fallback to historical safe data (Section~\ref{sec: policy_mod}).

% \begin{equation}
%     \alpha_i = \min_{x' \in E, x' \neq x_i} d(x', x_i) - \min_{y' \in D_S} d(y', x_i).
% \end{equation}
% With (squared) Euclidean distance this can be done with two KD trees. Here $D_S$ consists of $M$ safe states extracted from the set of safe trajectories $D_{\overline{A}}$ (a subset of all safe states may be used for faster runtime).

In Section~\ref{sec: nncp}, we used $C(\epsilon)$ to reason about safety of an individual state. However, at runtime we repeatedly query the warning system at each state in the ongoing trajectory. Although successive trajectory states are not IID, the error states in $D$ are independent, and if the current state is an error state it will be IID with $D$. Consequently, the single-state guarantee lifts to a multi-timestep trajectory guarantee:

% Our previous results in Sec.~\ref{sec: nncp} held for IID data and a single query. At runtime, we repeatedly query the warning system. Nevertheless, the states in $D$ are independent and if the current state is an error state it will be exchangeable with $D$. Thus, our warning system retains the miss rate guarantee in the multi-timestep setting, which we formally show in Theorem~\ref{thm: warning_miss}.

\begin{theorem}[Warning Miss Rate]
\label{thm: warning_miss}
Let $\tau$ be a new trajectory generated by executing $\pi$. 
\begin{equation}
\label{eq: miss_rate_informal}
\Pr(\mbox{$\tau$ yields no alert} \mid \tau \mbox{ reaches $\cur{A}$}) \leq \epsilon.
\end{equation}
\end{theorem}
\begin{proof}
Since trajectories are terminated upon reaching $\cur{A}$, $\tau$ reaches $\cur{A}$ if and only if $\tau(-1) \in \cur{A}$. 
$\tau$ yields no alert when $\forall i: \tau(i) \notin C(\epsilon)$. Thus, the left side of Eq.~\ref{eq: miss_rate_informal} is equivalent to 
\begin{align}
\Pr(\tau(i) \notin C(\epsilon) \ \forall i \mid \tau(-1) \in \cur{A}).
\end{align}

No alert at all times implies no alert at the final time hence
\begin{align}
    \Pr(\tau(i) \notin C(\epsilon) \ \forall i \mid \tau(-1) \in \cur{A}) \leq \\ \nonumber \Pr(\tau(-1) \notin C(\epsilon) \mid \tau(-1) \in \cur{A})
\end{align}

By the conformal guarantee, 
\begin{equation}
    \Pr(\tau(-1) \notin C(\epsilon) \mid \tau(-1) \in \cur{A}) \leq \epsilon.
\end{equation}
\end{proof}

% Theorem~\ref{thm: warning_miss} was inclusive meaning that if we reach an error state and alert at that time this does not count as a failure to alert. This is acceptable if we assume error states preempt failure i.e., shortly before, not during, collision. 
As a corollary, we can bound the overall probability $\pi$ becomes unsafe without warning.

\begin{corollary}
\label{cor: warning_error_rate}
With $\beta \coloneqq \Pr(\tau \mbox{ reaches $\cur{A}$})$,
\begin{equation}
\Pr(\tau \mbox{ reaches $\cur{A}$ without warning}) \leq \epsilon \beta.
\end{equation}
\end{corollary}
\begin{proof}
The left side equals 
\begin{gather}
\nonumber \Pr(\mbox{$\tau$ yields no alert} \mid \tau \mbox{ reaches $\cur{A}$}) \Pr(\tau \mbox{ reaches $\cur{A}$}) \leq \epsilon \beta
\end{gather}
\end{proof}

Using Corollary~\ref{cor: warning_error_rate} we can decide what $\epsilon$ to select. To enforce $\Pr(\tau \mbox{ reaches $\cur{A}$ without warning}) \leq \eta$, take $\epsilon \leq \eta / \hat{\beta}$ where $\hat{\beta} = N/P \approx \beta$ is the fraction of original unsafe rollouts (or a more conservative error rate bound~\cite{vincent2024guaranteesrobotperformanceusing}). Smaller $\epsilon$ tightens the miss rate, but can increase the number of false alarms (i.e., alerting during a safe trajectory).

% If we want $\Pr(\tau \mbox{ reaches $\cur{A}$ without warning}) \leq \eta$ for a user-specified $\eta$ we take $\epsilon$ small enough such that $\epsilon \hat{\beta} \leq \eta$. 

% That said, we can be assured that with enough data nearest neighbor classification will achieve no false alarms which we formalize in the below Theorem.

% \begin{theorem}[Asymptotic Accuracy]
%     \label{thm: asymp_cp}
%     As $N \rightarrow \infty$, the false alarm probability approaches $0$.
% \end{theorem}

% \begin{proof}
% By assumption 
% \end{proof}

% Note that Theorem \ref{thm: asymp_cp} required each state to be deterministically labeled either safe or unsafe i.e., noiselessly. This may be slightly violated in reality either due to human labeling error or if in representation learning the transformation $\phi(x)$ maps both a safe and unsafe state to the same point.

\section{Warning System Experiments}
\label{sec: warning_exp}

We evaluate the warning system in two simulated quadcopter settings: (i) an MPC policy with unknown obstacles and (ii) a visuomotor policy flying through a gate. We use this MPC also in our hardware experiments (Section~\ref{sec: policy_mod_exp}), and the visuomotor policy uses photorealistic images rendered from a Gaussian Splat model of the room used in hardware experiments.

% We evaluate the proposed warning system in two simulated quadcopter settings: 1. an MPC policy navigating in an environment with unknown obstacles and 2. a visuomotor policy navigating through a gate. While these are simulation experiments, we use the same MPC formulation in the hardware experiments in Section~\ref{sec: policy_mod_exp} and the visuomotor policy receives photorealistic image data rendered from a Gaussian Splat model of the room used in hardware experiments.

\subsection{Model Predictive Control Example}
\label{subsec: warning_MPC_sim}
We initialize the quadcopter at random starting states and use an MPC policy to steer to a given goal state. The 9-dimensional quadcopter state includes position $p = (p_x, p_y, p_z)$, Euler angles $\Theta = (\phi,\theta,\psi)$, and linear velocity $v = (\dot{p}_x,\dot{p}_y,\dot{p}_z)$. The policy outputs a 4-dimensional control action consisting of total thrust $F$ and body angular rates $\Omega$: $u = (F,\omega_x, \omega_y,\omega_z)$. At each step, MPC optimizes an open-loop plan $H$ steps into the future and applies the first resulting action $u_0$. We optimize using sequential convex programming (SCP), repeatedly solving a quadratic program (QP) obtained via affine approximation of the Euler-discretized, nonlinear dynamics about the previous intermediate solution $(\overline{x}, \overline{u})$:

\begin{equation}
\label{eq: traj_opt}
\scalebox{0.9}{$
\begin{gathered}
    \min_{u_{0:H-1}, x_{1:H}} \ \sum_{t=0}^{H-1} ||x_t - x_g||^2_{Q_t} + ||u_t - u_g||^2_{R_t} + ||x_H - x_g||^2_{Q_H} \\
    \text{s.t.} \ x_{t+1} = A_t x_t + B_t u_t + C_t, \quad \forall t = 0:H-1, \\
    F_u(t) u_t \leq g_u(t), \quad \forall t = 0:H-1, \\
    F_x(t) x_t \leq g_x(t), \quad \forall t = 1:H.
\end{gathered}
$}
\end{equation}

$x_0$ is the current state, $x_g$ the goal state, $u_g$ a reference action, $Q_t, R_t$ are positive-definite cost matrices, $A_t, B_t, C_t$ are found via affine approximation of $x_{t+1} = f(x_t, u_t)$ about $(\overline{x}_t, \overline{u}_t)$, $F_u(t), g_u(t)$, $F_x(t), g_x(t)$ specify control and state constraints.

As an illustrative example with a simple definition of safety, we add position obstacles unknown to the MPC policy and terminate trajectories upon collision. Figure~\ref{fig: warning_demo} shows example results of fitting the warning system. We treat position obstacles as unknown to MPC and terminate on collision. Figure~\ref{fig: warning_demo} shows (left) labeled demonstration trajectories (with $N = 25$ unsafe),~\footnote{We randomly initialize $x_0$ in a ring outside the four obstacles, set the goal state to hover at the room center, and plan with MPC horizon $H = 20$.} (middle) a 3-dimensional view of the unsafe-safe SUS region $C(\epsilon)$ ($\epsilon = 0.2$) which actually exists in the full 9-dimensional state space where it is learned,~\footnote{To visualize in 3 dimensions, we visualize a slice of each polytope $C_i$ by assuming $x$ matches the associated error state $x_i$ in all but position.} and (right) $50$ test rollouts with the warning system. As the original policy is often unsafe (crashing in $52\%$ of demonstrations), the warning system is necessarily often triggered, and no collisions result.

\begin{figure*}
     \centering
     \includegraphics[width=\textwidth]{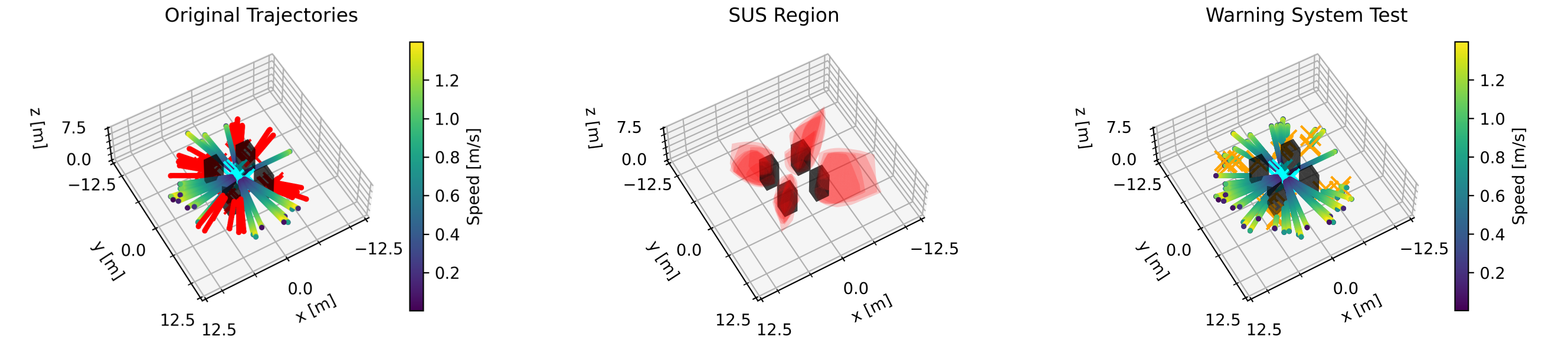}
     % \includesvg[width=\textwidth]{Figures/warning_demo_figs/warn_merged_horizontal.svg}
    \caption{Quadcopter MPC warning system for learned obstacle avoidance using $C(\epsilon)$ ($\epsilon = 0.2$, $N = 25$ unsafe samples). The left subfigure shows original trajectories, of which $52\%$ collide (marked in red). The middle subfigure shows a 3D visualization of the SUS region $C(\epsilon)$ which actually exists in the 9D state space. The right subfigure tests the warning system for 50 trajectories and no collisions occur. States triggering alert are marked with an orange `x.'}
    \label{fig: warning_demo}
\end{figure*}

To systematically assess our approach, we fit the warning system over $20$ repetitions, each time collecting $N = 25$ unsafe trajectories and a variable number of associated safe trajectories, and evaluate using a test set of $500$ trajectories. For fair comparison, all methods use the same data. 

% \begin{figure}
%      \raggedright
%      \subfloat[]{\includesvg[width=0.9\linewidth]{Figures/warning_systematic_figs/total_miss.svg}\label{fig: no_warn_and_unsafe}}
%      \hfill
%      \subfloat[]{\includesvg[width=0.9\linewidth]{Figures/warning_systematic_figs/omega.svg}\label{fig: omega}}
%     \caption{Systematic warning system results for a quadcopter MPC policy generated over $20$ repetitions. In each repetition, we fit our unsafe-safe conformal prediction warning system to new $N = 25$ unsafe trajectories and a variable number of associated safe trajectories. In Fig.~\ref{fig: no_warn_and_unsafe} and Fig.~\ref{fig: omega} we compare against performing conformal prediction using either just the safe or unsafe data. While all methods retain the error rate guarantee (Corollary~\ref{cor: warning_error_rate}), as assessed in Fig.~\ref{fig: no_warn_and_unsafe} using a test set of $500$ trajectories, the unsafe-safe approach yields fewer false alarms as shown in Fig.~\ref{fig: omega}.}
%     \label{fig: warning_systematic}
% \end{figure}

\begin{figure}
    \centering
     \includegraphics[width=0.9\linewidth]{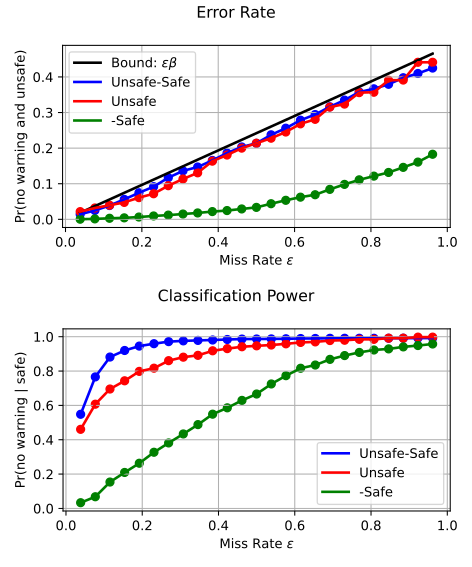}
    \caption{Systematic warning system results for a quadcopter MPC policy generated over $20$ repetitions. In each repetition, we fit our unsafe-safe conformal prediction warning system to new $N = 25$ unsafe trajectories and a variable number of associated safe trajectories. We compare against performing conformal prediction using either just the safe or unsafe data. While all methods retain the error rate guarantee (Corollary~\ref{cor: warning_error_rate}), as assessed in the top subfigure using a test set of $500$ trajectories, the unsafe-safe approach yields fewer false alarms as shown in the bottom subfigure.}
    \label{fig: warning_systematic}
\end{figure}

In Figure~\ref{fig: warning_systematic} (top) we verify that for $\epsilon \in [\frac{1}{N+1}, \frac{N}{N+1}]$, the fraction of test trajectories resulting in error with no alert is $\leq \epsilon \beta$ (Corollary~\ref{cor: warning_error_rate}). Besides the unsafe-safe nearest neighbor approach, using Eq.~\ref{eq: two_sample_special} (labeled Unsafe-Safe), we consider ablations only using distance to the nearest unsafe state: $s(D_F; x) = \min_{x' \in D_F} ||x' - x||^2$ (labeled Unsafe), or instead flagging if distance to the nearest safe state is too large: $s(D_F; x) = -\min_{y' \in D_G} ||y' - x||^2$ (labeled -Safe). In all cases Corollary~\ref{cor: warning_error_rate} holds as each approach is nearest neighbor conformal prediction but with a different dissimilarity measure.~\footnote{In the -Safe case, $s(D_F; x) = s(x)$ depends only on $x$, not the points $x'$ in $D_F$, so we can treat the dissimilarity measure itself as the score $d(x',x) = s(x)$ in Eq.~\ref{eq: score func} and Theorem~\ref{thm: main_cp} reduces to split conformal prediction.} In Figure~\ref{fig: warning_systematic} (bottom) we plot an ROC-like curve studying classification power as $\epsilon$ varies. The y-axis plots the complement to the false alarm rate i.e., $\Pr(\mbox{no warning} \mid \mbox{safe})$ and should remain high for an effective safety classifier. Across $\epsilon$, the unsafe-safe approach outperforms ablations. Using only safe data performs poorly as new initial conditions can be far from recorded safe trajectories, triggering false alarms. 

Because of problem setting differences, a direct comparison to many alternative approaches for learning safety (Section~\ref{sec: lit_rev}) is infeasible; such methods often rely on constraint-satisfying expert trajectories, learn safety interactively, or collect dense labels from expert intervention. In contrast, we focus on learning safety in a single step from a limited offline dataset with binary unsafe labels. We therefore consider widely used machine learning classifiers as a natural point of comparison, shown in Figure~\ref{fig: ml_baselines}. This comparison helps assess to what extent we can reliably predict safety in the limited data setting, illustrating the impact of data imbalance and data size on the performance of standard classifiers. Although the problem settings differ, several of the works in Section~\ref{sec: lit_rev} also train machine learning classifiers to predict safety, making this comparison relevant within the robotics context.

As baselines in Figure~\ref{fig: ml_baselines}, we use a random forest (RF), support vector machine (SVM), and neural network (NN), each with/without class balancing (which equally weights the total loss contribution from unsafe and safe samples). We fit on the entire data and plot an `x' at the point $x = \Pr(\mbox{no warning} \mid \mbox{ unsafe})$, $y = \Pr(\mbox{no warning} \mid \mbox{ safe})$. These standard techniques do not allow us to control for the miss rate, $\epsilon$, so only one `x' appears per method. To better compare with our method, we also calibrated each baseline using split conformal prediction (see Section~\ref{sec: cp_overview}), using $10$ unsafe trajectories to calibrate the alert threshold and the remaining $15$ for model fitting.~\footnote{We reserve, in particular, $N = 10$ unsafe samples to achieve sufficient resolution for the resulting conformal bounds: $k/(N+1), k = 1, .., N$.} Unlike our method, full conformal prediction with these baselines is not closed-form and is thus intractable, requiring $N+1$ model fits for each new query state. For method comparison, we plot ROC-like curves as a function of $\epsilon$ in Figure~\ref{fig: ml_baselines}. Across $\epsilon$, our nearest neighbor conformal approach outperforms all but the random forest with class-balancing, which performs similarly. 

Despite similar performance, our approach has some interpretability and theoretical benefits compared with the random forest in this limited data setting. Our approach avoids parameter tuning (e.g., random forest tree depth) and provides an explicit geometric description for the learned SUS region (which we later exploit for policy modification). From a theoretical perspective, conformal prediction is performed in our method using all $N = 25$ error states, reusing the same data for classification and calibration. In contrast, for the random forest, we must reserve a held-out calibration set of smaller size $N = 10$ for conformal prediction. While conformal prediction guarantees average coverage at least $1 - \epsilon$, increasing $N$ reduces the minimum enforceable miss rate $\epsilon = 1/(N+1)$, reduces overcoverage bounds (Theorems~\ref{thm: overcoverage_sym},~\ref{thm: overcoverage_asym}), and shifts the coverage distribution closer to $1 - \epsilon$ (see Figure~\ref{fig: hist_cov}).

Of course, nearest neighbor conformal prediction will not always outperform standard and powerful machine learning alternatives. Rather, we propose it as a simple, interpretable, and data-efficient approach that provides theoretical guarantees without parameter tuning or separate calibration data.

\begin{figure}
    \centering
    \includegraphics[width=0.9\linewidth]{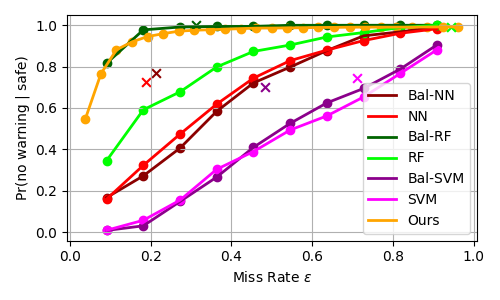}
    \caption{Comparison of our warning system against machine learning baselines. Using the same data as in Fig.~\ref{fig: warning_systematic}, we repeatedly fit several machine learning baselines (random forests (RF), neural networks (NN), support vector machines (SVM)) with and without class balancing, shown as `x's on the plot. Furthermore, to control the miss rate $\epsilon$ of these classifiers, we use split conformal prediction which requires holding out calibration data, unlike in our method. Curves show the ROC performance of each method across $\epsilon$. Despite its simplicity, our approach performs similarly to random forest with class-balancing and outperforms all the other methods.}
    \label{fig: ml_baselines}
\end{figure}

\subsection{Visuomotor Quadcopter Policy Example}

\begin{figure}
     \centering
     \includegraphics[width=\linewidth]{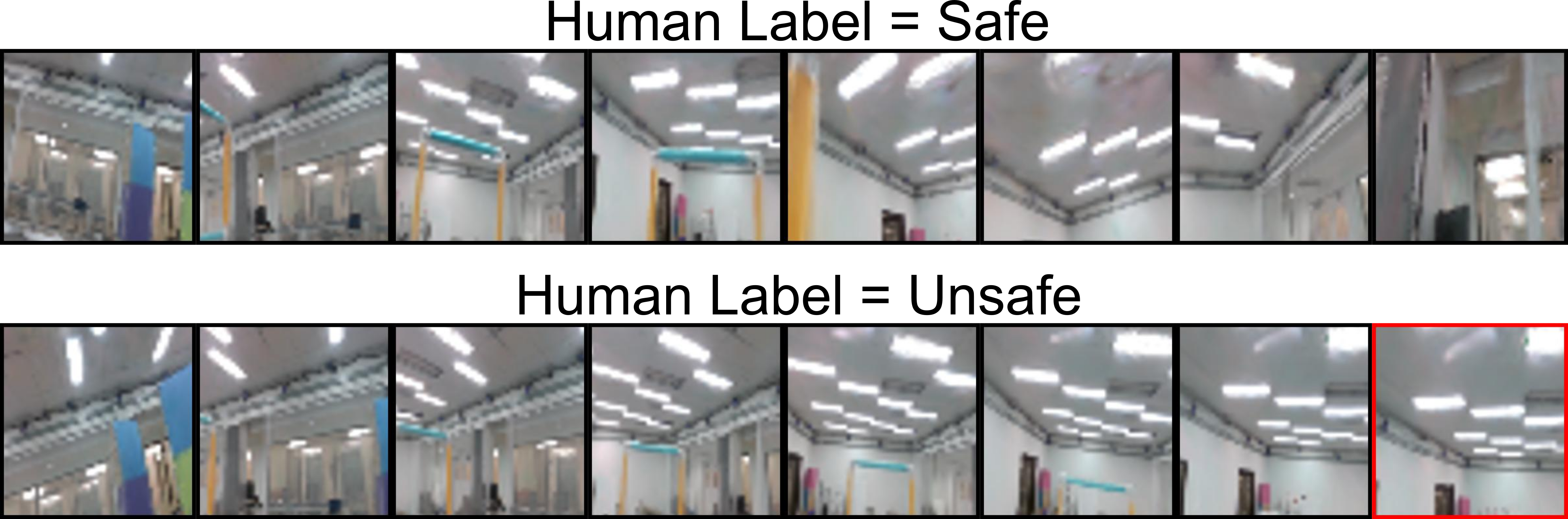}
    \caption{Example trajectories used for fitting the warning system to a visuomotor quadcopter policy. The top row shows frames from a safe trajectory that successfully passes through the gate and is not stopped by the human. The bottom row shows frames from an unsafe trajectory that is terminated by the human labeler when it is deemed the quadcopter will not pass through the gate. The final image where the human triggers termination is outlined in red.}
    \label{fig: visuo_train_rollout}
\end{figure}

We also developed our warning system to augment a visuomotor policy tasked with flying a quadcopter through a designated gate. The policy we use is an early version of that developed in \cite{LowEtAlRAL24_SousVide}. The policy is parameterized as a neural network that takes as input an RGB image, a history of partial state and action information, and outputs normalized thrust and body angular rates. To generate training data for the policy, a Gaussian Splat \cite{kerbl20233d, Ye_gsplat, nerfstudio} representation of the physical environment is fit using real image data to create a visually realistic simulator. Specifically, the Gaussian Splat represents the scene as a set of 3D Gaussian primitives whose parameters (position, covariance, opacity, and color) are optimized from multi-view real images. Novel RGB views are then rendered by rasterizing and compositing these Gaussians in pixel space, enabling photorealistic reconstruction of the environment from arbitrary viewpoints. The policy is trained via behavior cloning to mimic a model-based trajectory optimizer which has privileged information regarding state and parameters (mass, thrust coefficient) that are randomized at initialization. % For more details, we refer readers to the Appendix and the original work.
% In simulated training and experiments, the quadcopter mass and thrust coefficient are randomly offset. Thus, at runtime, the visuomotor policy must be able to account for these parameter changes, inferring based on the partial state and action history input to the neural network. 

To fit the warning system, we simulated $200$ trajectories while randomizing starting state/parameters and watched replayed videos. Using a simple user interface, a single annotator (an author) terminated any trajectories if they believed the policy had veered too far and was unlikely to pass through the gate, yielding $45$ unsafe trajectories. We stored $50 \times 50 \times 3$ RGB image data for use in the conformal warning system. Figure~\ref{fig: visuo_train_rollout} shows images from an example safe and unsafe trajectory, as labeled by the human. 

% \begin{figure}
%      \centering
%      \subfloat{\includesvg[width=\linewidth]{Figures/visuomotor_demo_figs/train_safe_rollout.svg}}
%      \hfill
%      \subfloat{\includesvg[width=\linewidth]{Figures/visuomotor_demo_figs/train_unsafe_rollout.svg}}
%      \hfill
%     \caption{Example trajectories rendered in a Gaussian Splat simulator used in fitting the warning system for a visuomotor quadcopter policy. The top row shows downsampled frames from a safe trajectory that successfully passes through the gate and is not stopped by the human. The bottom row shows downsampled frames from an unsafe trajectory that is terminated by the human labeler when it is deemed the quadcopter will not pass through the gate. The final image where the human triggers termination is outlined in red.}
%     \label{fig: visuo_train_rollout}
% \end{figure}

\begin{figure}
     \centering
     \includegraphics[width=0.9\linewidth]{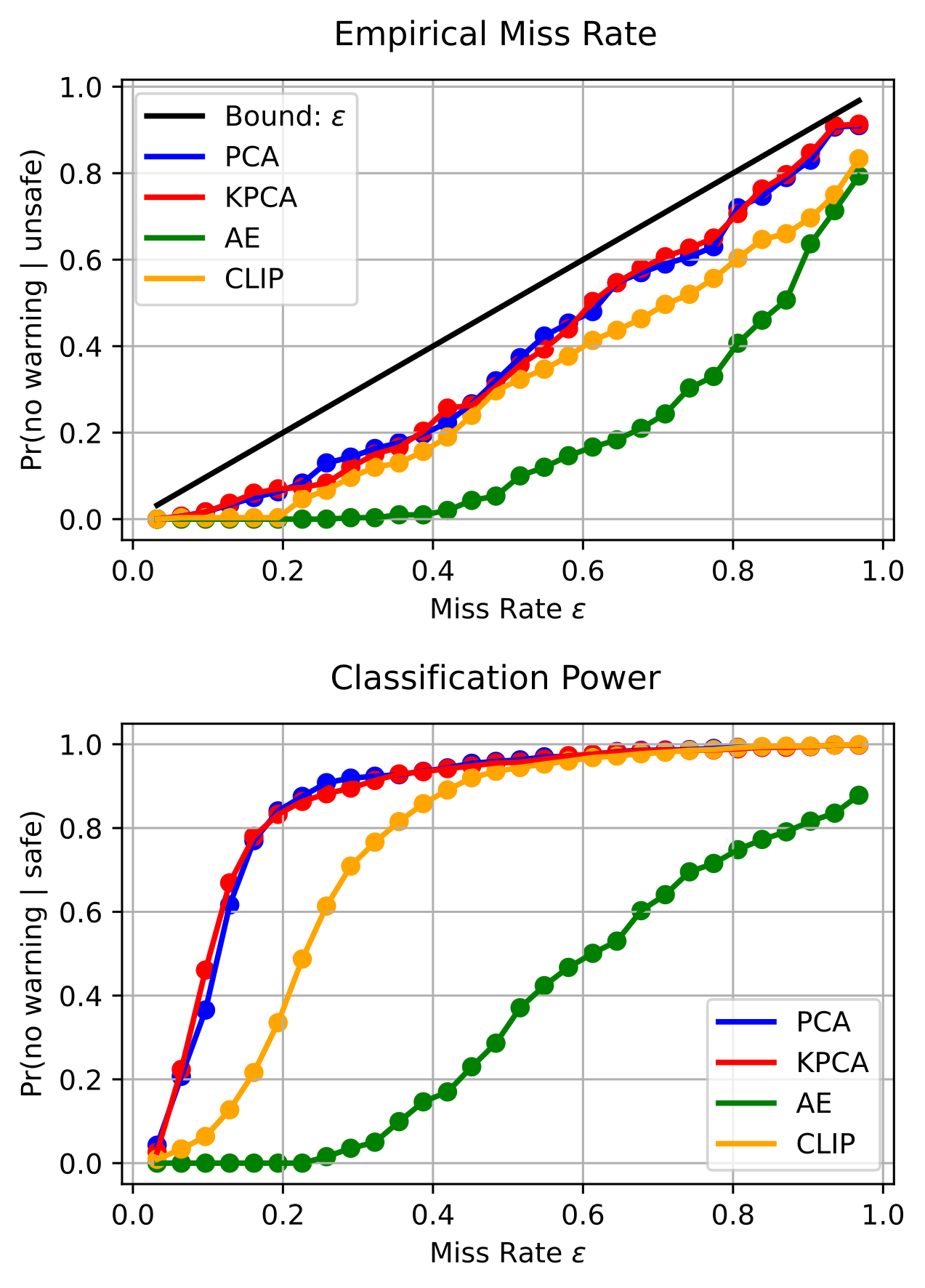}
    \caption{Results for conformal warning systems applied to a visuomotor quadcopter policy with different image representations. The results average across $20$ train-test splits of $200$ trajectories. In each split, $30$ unsafe trajectories and $100$ safe trajectories are used to fit the conformal warning system. The safe trajectories are also used to fit the representation model (except pre-trained CLIP which is fixed). Both PCA and KPCA perform well, but all methods retain Theorem~\ref{thm: warning_miss}'s miss rate guarantee.}
    \label{fig: visuomotor_systematic}
\end{figure}

% \label{fig: train_safe_rollout}}
% \label{fig: train_unsafe_rollout}

Instead of directly using images in the nearest neighbor procedure, we reuse the safe trajectories to learn a $400$ dimensional latent representation. We fit PCA, KPCA, and neural network autoencoder models to learn the transformation as well as a pretrained CLIP model~\cite{clip}. We compare simple models (PCA, KPCA) to expressive ones (autoencoder) to assess the effect of model capacity under limited data. Richer representations can help the nearest neighbor classifier better distinguish safe from unsafe data, but fitting these representations can require more data to generalize. We include a pretrained/frozen model (CLIP) as a potential method for obtaining expressive features without fitting to our small dataset. 

% These comparisons illustrate our method's compatibility with a variety of representation learning techniques.

Over $20$ repetitions we select a random subset of $30$ unsafe and $100$ safe trajectories to fit the warning system and assess performance using the remaining trajectories. Even with the $400$-dimensional latent space, the per-query runtimes, including transformation via PCA/KPCA/autoencoder/CLIP, remain low, averaging $4.9$, $23.8$, $0.4$, and $3.0$ milliseconds respectively. Figure~\ref{fig: visuomotor_systematic} (top) shows the average warning system miss rate for varying $\epsilon$. Regardless of transformation, the empirical average falls below the $\epsilon$ value as guaranteed in Theorem~\ref{thm: warning_miss}. In Figure~\ref{fig: visuomotor_systematic} (bottom), we show the average empirical estimate of $\Pr(\mbox{no warning} \mid \mbox{safe})$ for each warning system. PCA/KPCA perform best with the limited data, outperforming the autoencoder and providing a task-specific representation more effective than pre-trained CLIP. With more data, higher-capacity models may become advantageous. % Further details in the Appendix.

% \begin{figure}
%      \centering
%      \subfloat[]{\includesvg[width=\linewidth]{Figures/visuomotor_systematic_figs/conditional_miss.svg}}\label{fig: visuo_miss_rate}
%      \hfill
%      \subfloat[]{\includesvg[width=\linewidth]{Figures/visuomotor_systematic_figs/visuo_omega.svg}\label{fig: visuo_omega}}
%      \hfill
%     \caption{Results for conformal warning systems applied to a visuomotor quadcopter policy with different image representations. The results average across $20$ train-test splits of $200$ trajectories. In each split, $30$ unsafe trajectories and $100$ safe trajectories are used to fit the conformal warning system. The safe trajectories are also used to fit the representation model (except pre-trained CLIP which is fixed). Both PCA and KPCA perform well, but all methods retain Theorem~\ref{thm: warning_miss}'s miss rate guarantee.}
%     \label{fig: visuomotor_systematic}
% \end{figure}

For one warning system fit, Figure~\ref{fig: visuo_test_rollout} shows frames from test safe and unsafe trajectories. We visualize the associated $p$-value as a border color with warmer colors indicating higher values, i.e., images deemed more unsafe by the warning system. In the (top) safe trajectory the $p$-values remain low while in the (bottom) unsafe trajectory the $p$-values increase as the quadcopter starts flying too high above the gate. Using the $p$-value as a calibrated safety measure, the warning system can thus act as a runtime safety monitor.

\begin{figure}
     \centering
     \includegraphics[width=\linewidth]{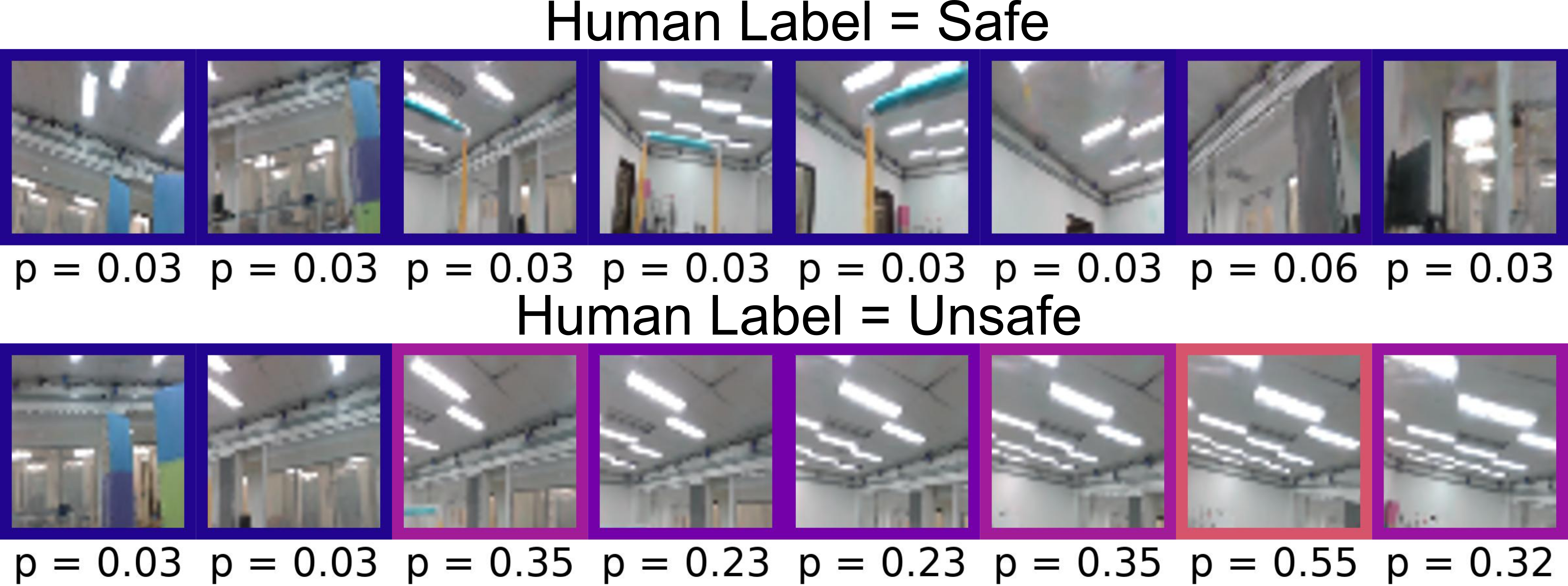}
    \caption{Test trajectories in a Gaussian Splat simulator with associated warning system $p$-values. The top row shows a safe trajectory with low $p$-values throughout. The bottom row shows an unsafe trajectory where the $p$-values increase as the quadcopter begins to fly above the gate, not through as desired.}
    \label{fig: visuo_test_rollout}
\end{figure}

Although the visuomotor policy is limited to one environment, these experiments demonstrate our approach can scale to learning safety from high-dimensional image data, particularly when coupled with feature transformation. Using visual features makes the SUS region independent of global coordinates, providing a way to transfer to multiple environments.

% As long as the collected safety data is drawn from the distribution over possible environments/conditions, the conformal prediction and warning system guarantees presented would directly apply.

% \begin{figure}
%      \centering
%      \subfloat{\includesvg[width=\linewidth]{Figures/visuomotor_demo_figs/test_safe_rollout.svg}}
%      \hfill
%      \subfloat{\includesvg[width=\linewidth]{Figures/visuomotor_demo_figs/test_unsafe_rollout.svg}}
%      \hfill
%     \caption{Test trajectories in a Gaussian Splat simulator with associated warning system $p$-values. The top row shows a safe trajectory with low $p$-values throughout. The bottom row shows an unsafe trajectory where the $p$-values increase as the quadcopter begins to fly above the gate, not through as desired. 
%     \label{fig: visuo_test_rollout}
% \end{figure}

% \label{fig: test_safe_rollout}
% \label{fig: test_unsafe_rollout}

% \begin{figure*}
%      \centering
%      \includesvg[width=\textwidth]{Figures/visuomotor_demo_figs/visuo_all_merged.svg}
%     \label{fig: visuo_demo}
% \end{figure*}
\section{Distribution Shift and Multiple Labelers}
\label{sec: dist_shift}

Previous sections assumed IID conditions between data collection and policy deployment, ensuring $1-\epsilon$ coverage (Theorem~\ref{thm: main_cp}). Here, we provide theoretical and empirical analysis for how distribution shift affects SUS region coverage.

% Section~\ref{sec: nncp} and Section~\ref{sec: warning_sys} assumed conditions during data collection, when presenting $\pi$ to a human labeler, are unchanging and match deployment conditions. This ensures the error states $D$ and a new test-time error are IID, so conformal prediction guarantees $1 - \epsilon$ coverage (Theorem~\ref{thm: main_cp}). In this section, we discuss the effect of distribution shift (i.e., violating the IID assumption) on SUS region coverage.

Firstly, distribution shift may occur internally while demonstrating $\pi$ to the human, making error states in $D$ not IID. This could occur if the data are collected by multiple labelers (for simplicity two), with different safety definitions. To detect inconsistency, one can first form SUS regions $C_1(\epsilon), C_2(\epsilon)$ separately using the data $D_1, D_2$ produced by each labeler, and then check how well each covers the other labeler's data (i.e., checking if the fraction of contained points is near $1 - \epsilon$). If both SUS regions maintain coverage, the data can be pooled to form a joint $C(\epsilon)$. Otherwise, it suggests the labelers disagree on their safety definition, and instead of merging the data, we can keep the separate SUS regions (effectively group-balanced conformal prediction~\cite{angelopoulos2022gentle, vovk2012conditional}). $C_1(\epsilon)$ and $C_2(\epsilon)$ then achieve $1 - \epsilon$ coverage for their respective labelers and the union $C_1(\epsilon) \cup C_2(\epsilon)$ simultaneously satisfies both labelers. 

To check for distribution shift on a point-by-point basis, one can also monitor the sequence of conformal $p$-values (see subsection~\ref{subsec: p_val}) when comparing each new datum to the already collected data~\cite{angelopoulos2024theoretical}. A growing supermartingale indicates distribution shift and can be used to prompt data splitting i.e., we can build $C_1(\epsilon)$ using the already collected data and treat new points as dataset $D_2$.

There can also be distribution shift between data collection and deployment, e.g., if the initial condition distribution $\cur{D}_0$ or dynamics $\cur{T}$ change. To analyze the impact on $C(\epsilon)$ coverage, the source of distribution shift is irrelevant; we need only consider how such a change shifts the resulting distribution over error states i.e., on the policy error distribution $F$ to a modified $\Tilde{F}$. Directly from existing conformal prediction results under distribution shift \cite{barber2023conformal}, we obtain:

\begin{theorem}[Coverage Loss under Distribution Shift]
\label{thm: dist_shift}
Let $D = \{x_1,...,x_N\}$ be states drawn IID from $F$ and let $x_{N+1} \sim \Tilde{F}$ be independently sampled from a modified distribution. Then, for $C(\epsilon)$ constructed as in Theorem~\ref{thm: main_cp},
\begin{equation}
    \Pr[x_{N+1} \in C(\epsilon)] \leq 1 - \epsilon - \frac{N}{N+1} (2 d_{TV}(F, \Tilde{F}) - d_{TV}(F, \Tilde{F})^2)
\end{equation}
where $d_{TV}(F, \Tilde{F})$ denotes total variation distance.
\end{theorem}
This coverage loss grows monotonically with $d_{TV}(F,\Tilde{F})$. Substituting $\Tilde{\epsilon} = \epsilon + \frac{N}{N+1} (2 d_{TV}(F, \Tilde{F}) - d_{TV}(F, \Tilde{F})^2)$ into Theorem~\ref{thm: warning_miss} and Corollary~\ref{cor: warning_error_rate} yields trajectory-level guarantees under distribution shift.

Because coverage loss depends only on the distribution shift between policy error distributions, $C(\epsilon)$ is robust to shifts in initial conditions or dynamics that leave $F$ largely unchanged (e.g., starting farther from obstacles). Robustness can be further improved by applying transformations $\phi$ which make $\phi(x)$ less sensitive to a distribution shift over $F$ (e.g., projecting to just position makes $C(\epsilon)$ insensitive to velocity changes). By the data-processing inequality~\cite{polyanskiy2015strong}:
\begin{equation}
\label{eq: dpi_phi}
    d_{TV}(F_{\phi}, \Tilde{F}_{\phi}) \leq d_{TV}(F, \Tilde{F})
\end{equation}
so building $C(\epsilon)$ using transformed states $\{\phi(x_i)\}_{i=1}^N$ yields a coverage loss bound no worse than if building $C(\epsilon)$ with the original $D$.~\footnote{$F_{\phi}$ is the distribution for $\phi(x), x \sim F$ (similarly for $\Tilde{F}_\phi$ and $\Tilde{F}$).} Of course, a poor choice of $\phi$ may destroy key differences between unsafe and safe data, resulting in a conservative $C(\epsilon)$ and many false alarms.

Distribution shifts can also change the relative frequency of different failure modes (e.g., how often the quadcopter collides with each pillar in Figure~\ref{fig: warning_demo}). Even without knowing each component $F_i$ of the error distribution $F$, we can bound the total variation distance, and thus the conformal coverage gap:

\begin{theorem}[Mixture Distribution Total Variation Bound]
\label{thm: mix_shift}
Let $F = \sum_{k=1}^m w_k F_k$ be a mixture distribution, with components $\{F_k\}_{k=1}^m$ and associated nonnegative weights $w = (w_1,...,w_m)$ summing to $1$. Let $\Tilde{F} = \sum_{k=1}^m \Tilde{w}_k F_k$ have same components but different weights $\Tilde{w} = (\Tilde{w}_1,...,\Tilde{w}_m)$. Then
\begin{equation}
    d_{TV}(F, \Tilde{F}) \leq 1/2 ||w - \Tilde{w}||_1.
\end{equation}
\end{theorem}

In Figure~\ref{fig: shift_exp}, we reuse the setup of Figure~\ref{fig: warning_demo} but change the initial condition distribution, varying how often we start on either side of the room, which varies the frequency of collisions with upper $y > 0$ versus lower pillars (failure modes $F_1, F_2$ respectively). We build $C(\epsilon)$ under different data collection conditions: generating $D$ with mixture $w = (0,1)$ starting only with $y \leq 0$, a skewed split $w = (0.25, 0.75)$, and even split $w = (0.5,0.5)$. We deploy under $5$ test distributions, varying $\Tilde{w}_1 \in [0,1]$, and plot the resulting miss rate. When both modes are present in $D$ ($w_1 = 0.25, 0.5$), $C(\epsilon)$ remains robust across deployment mixtures (retaining miss rate close to nominal $\epsilon = 0.05$). If the collected data omits a mode ($w_1 = 0$), coverage deteriorates as the unseen mode becomes more frequent and the bound (Theorem~\ref{thm: dist_shift}) is fairly tight. Thus, even for novel failure modes (e.g., unseen obstacles), which are absent from the collected data and consequently missed by the warning system, we can still characterize the miss rate increase.

Theorem~\ref{thm: mix_shift} is also relevant if we decide to forgo data splitting and form a joint $C(\epsilon)$ for multiple distinct labelers; we can view each labeler $k$ as contributing a component $F_k$.

\begin{figure}
     \centering
     \includegraphics[width=0.9\linewidth]{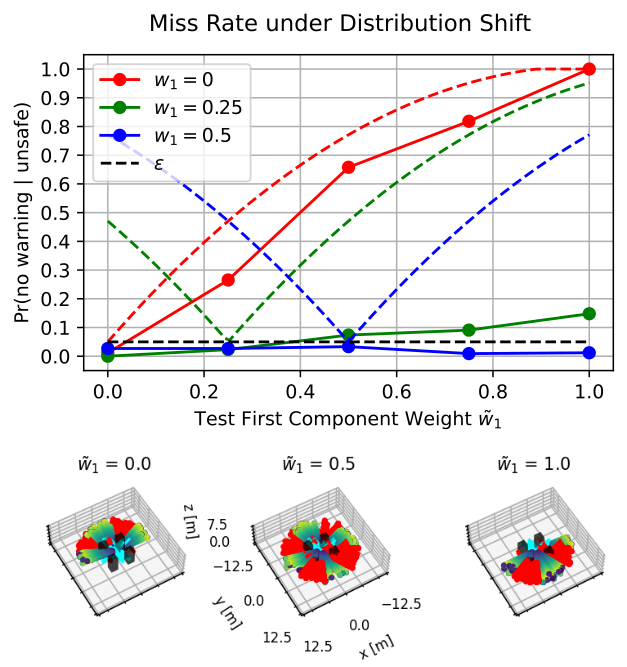}
    \caption{Distribution shift impact on warning system miss rate. We change the frequency $\Tilde{w}_1$ of starting in the upper half of the room (visualized for $\Tilde{w}_1 = 0, 0.5, 1$ in the bottom subfigure). We build three SUS regions $C(\epsilon)$ (with $\epsilon = 0.05$, $N = 25$, and unsafe-safe score) under different data collection conditions: $w_1 = 0, 0.25, 0.5$. For each $C(\epsilon)$, the top subfigure shows the miss rate (evaluated from $100$ test trajectories, over $10$ fitting repetitions) of the resulting warning system across test conditions (different $\Tilde{w}_1$). In each case, we overlay the theoretical bound (Theorems~\ref{thm: dist_shift},~\ref{thm: mix_shift}) as a similarly colored dashed line. For $w_1 = 0$ the miss rate increases significantly and the bound is fairly tight but for $w_1 = 0.25, 0.5$ the miss rate remains near $\epsilon$, illustrating that $C(\epsilon)$ can remain robust to changes in frequency of different failure modes, so long as they are represented in the collected data.}
    \label{fig: shift_exp}
\end{figure}

% In fact, as a general result, the distribution shift of the policy error distribution must be smaller than the original distribution shift imposed by changing the initial condition distribution (from $\cur{D}_0$ to $\Tilde{\cur{D}}_0$) i.e., 
% \begin{equation}
%     d_{TV}(F, \Tilde{F}) \leq d_{TV}(\cur{D}_0, \Tilde{\cur{D}}_0)
% \end{equation}
% where this follows directly from the data-processing inequality \cite{polyanskiy2015strong}, which states that applying a (possibly random) transformation to two distributions (in our case, mapping from the distribution over initial states $x_0$ to the distribution over associated error states in $\cur{A}$), can only decrease the $f$-divergence (of which total variation distance is a special case \cite{polyanskiy2015strong}) between the distributions.
\section{Policy Modification}
\label{sec: policy_mod}

In this section, we use the warning system (see Section~\ref{sec: warning_sys}) to modify the original policy and improve its safety. Our approach is to avoid the SUS region $C(\epsilon)$, using the warning system to trigger a backup safety controller that steers to historical safe data. We focus on the MPC setting as we will exploit the ability to 1. look ahead into the future using the open-loop plan and 2. easily add/enforce constraints by modifying trajectory optimization.

Besides $\pi$, we use an additional backup policy $\pi_B$ to track to historical safe trajectories. Offline, we search through the set of safe collected trajectories $D_{\overline{\cur{A}}}$ and identify the ``backup" subset, denoted $D_B$, which avoid $C(\epsilon)$. At runtime, we start by using $\pi$, repeatedly solving the optimization in Eq.~\ref{eq: traj_opt}. If any state in the open-loop plan $x_{1:H}$ enters $C(\epsilon)$ i.e., triggers a warning, we record the future alert timestep $t_w \in \{1, ..., H\}$ and switch to the backup safety mode. When first switching, we identify the nearest safe trajectory $\tau_B \in D_B$ containing state $x_B$, at index $t_B$, closest to the current state $x_0$. Our backup policy $\pi_B$ will aim to drive the current trajectory towards $\tau_B$. At each time while in backup safety mode, we increment $t_B$ and extract the relevant safe states from $\tau_B$ as tracking targets: $x_B^0 = \tau_B(t_B), ..., x_B^H = \tau_B(t_B+H)$. $\pi_B$ solves a similar optimization to Eq.~\ref{eq: traj_opt} except instead of a single goal state $x_g$, we track the selected backup states $\{x_B^t\}_{t=0}^H$. We release the backup safety mode and return to using $\pi$ if either (i) the current state $x_0$ is close to $\tau_B$ and at least $t_w$ timesteps have elapsed, or (ii) a maximum number of iterations have elapsed. Since we do not track $\tau_B$ indefinitely, our approach applies even if the goal region changes run to run, as shown in our hardware experiments.

% solving the below optimization and applying $u_0$:
% \begin{equation}
% \label{eq: track_traj_opt}
% \scalebox{0.9}{$
% \begin{gathered}
%     \min_{u_{0:H-1}, x_{1:H}} \ \sum_{t=0}^{H-1} ||x_t - x_B^t||^2_{Q_t} + ||u_t - u_g||^2_{R_t} + ||x_H - x_B^H||^2_{Q_H} \\
%     \text{s.t.} \ x_{t+1} = A_t x_t + B_t u_t + C_t, \quad \forall t = 0:H-1, \\
%     F_u(t) u_t \leq g_u(t), \quad \forall t = 0:H-1, \\
%     F_x(t) x_t \leq g_x(t), \quad \forall t = 1:H.
% \end{gathered}
% $}
% \end{equation}

We add state constraints to ensure $\pi_B$ avoids the SUS region $C(\epsilon)$ while steering to $\tau_B$. Using the SUS region's geometry (Theorem~\ref{thm: union_cp}) as a union of balls/polyhedra $\cup_{i=1}^N C_i$, we treat $C(\epsilon)$ as a set of learned obstacles which we constrain $\pi_B$ to avoid. We include these constraints in SCP by conservatively approximating $C_i$ using the previous guess, selecting one halfspace constraint per $C_i$ to enforce avoidance.~\footnote{To remain outside $C_i = \{x \mid A_i x \leq b_i\}$ we enforce $a_{ij}^T x > b_j$, selecting the farthest in signed distance from previous guess $\overline{x}_t$: $j^* = \arg \max_{j=1:M} d_{ij}$, $d_{ij} = (a_{ij}^T \overline{x}_t - b_{ij}) / ||a_{ij}||_2$.} 

% ~\footnote{To improve runtime, we can reduce the number of halfspaces defining $C_i$ by subselecting safe states $D_G$ in Eq.~\ref{eq: two_sample_special} or by pruning redundant constraints.}

% As a practical concession to the over-approximation of each $C_i$, we use a slack variable $\delta$ and require $A_x(t) x_t \leq b_x(t) + \delta$ but add a slack penalty $c_\delta |\delta|$ for user-determined $c_\delta > 0$.

We summarize our approach to policy modification below:
\begin{enumerate}
    \item Offline, collect a backup dataset $D_B$ of safe trajectories which do not enter $C(\epsilon)$. 
    \item Online, use $\pi$ but switch to the backup safety mode if the look-ahead plan $x_{1:H}$ enters $C(\epsilon)$ (at a timestep $t_w$).
    \item When first switching to the backup safety mode, find the nearest trajectory $\tau_B$ to the current state $x_0$ with nearest state $x_B \in \tau_B$ achieved at index $t_B$.
    \item While in the backup safety mode, increment $t_B$ and apply $\pi_B$ to steer to tracking targets $\{x_B^t)\}_{0:H} \in \tau_B$. Also, add linear constraints to avoid $C(\epsilon)$, one per $C_i$.
    \item Release the backup safety mode after a maximum elapsed time or if the current state $x_0$ is now close enough to $\tau_B$ and the original alert time $t_w$ has passed.
\end{enumerate}

A straightforward approach to improving the safety of $\pi$ would be to add constraints to avoid $C(\epsilon)$ at all times and forgo using a backup policy. However, modifying $\pi$ to avoid $C(\epsilon)$ causes a distribution shift over the visited states, and in particular, significant distribution shift over the resulting error distribution $F$. By design $C(\epsilon)$ covers none of the new errors under the constrained policy and new errors in the region $\cur{A} \setminus C(\epsilon)$ can yield potentially high error rates (as shown in Section~\ref{sec: policy_mod_exp}). By nominally using $\pi$ and switching to $\pi_B$ upon alert, there is no distribution shift until alert. This keeps the warning system calibrated, reliably detecting unsafe states in the look-ahead plan. Because we switch to $\pi_B$ preemptively using alerts in the look-ahead plan, we can hope there is time for a well-tuned $\pi_B$ to steer to $\tau_B$ before reaching an unsafe state. Implicitly, this requires diverse backup trajectories in $D_B$ so that $\tau_B$ is close enough to be reached using $\pi_B$.

% Lastly, one might consider simply tracking to the historical safe data at all times but this may be overly conservative, requiring inefficient detouring, and is similar to using our approach with small $\epsilon$.

% While likely safer, this approach could be overly conservative, requiring unnecessary detouring in cases where the original policy $\pi$ would have acted safely. \footnote{Furthermore, always tracking to historical safe data cannot be used successfully in the multi-goal setting as the backup trajectory $\tau_B$ would have a different goal than the current $x_g$.} In some sense, this is an extreme of our method since as $\epsilon$ is decreased our backup safety mode will be triggered more frequently. 

% In the unsafe-only case each $C_i = \{x \mid (x - x_i)^T (x - x_i) \leq r\}$. Therefore, to avoid $C(\epsilon)$ a given state $x$ must satisfy $h_i(x) = 1/r(x - x_i)^T (x - x_i) - 1 > 0$ for $i=1,...,N$. In our iterative SCP approach, for each state $x_t$ we simply affinize each constraint $h_i(x)$ about the previous guess $\overline{x}_t$ to obtain a halfspace constraint $a_i^T x_t \leq b_i$ where $a_i = -2/r (\overline{x}_t - x_i)$ and $b_i = h_i(\overline{x}_t) + a_i^T \overline{x}_t$ \footnote{Including/excluding the boundary should have minimal effect}.

\section{Policy Modification Experiments}
\label{sec: policy_mod_exp}

We perform simulation and hardware experiments to evaluate our approach for policy modification. Our simulation experiments extend Section~\ref{sec: warning_exp}; we now show the modified policy can safely navigate to the goal. For our hardware experiments, we use human video feedback collected with a Gaussian Splat simulator to modify a quadcopter MPC policy to navigate more cautiously in the test environment.

\subsection{Simulation Experiments}

\begin{figure*}
     \centering
    \includegraphics[width=\textwidth]{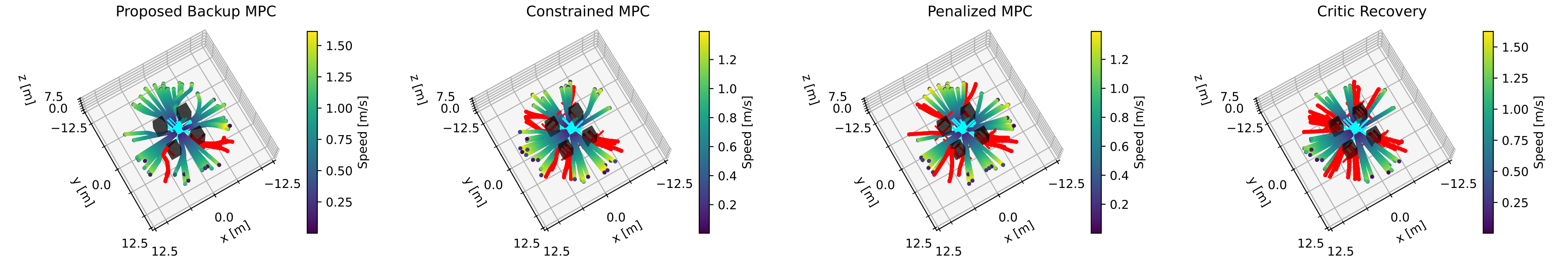}
    \caption{Approaches for modifying a MPC policy using collected safety data. From left to right: our backup MPC, constrained MPC avoiding $C(\epsilon)$, penalized MPC with slack on avoiding $C(\epsilon)$, and a safety critic approach from \cite{Thananjeyan2021} that switches to minimizing risk when $\hat{Q}(x_t, u_t) > \epsilon$. We set $\epsilon = 0.1$ and collect $N = 25$ errors. Each subfigure shows $50$ trajectories with collision rates of $14\%$, $32\%$, $30\%$, and $50\%$ (vs. $52\%$ for the unmodified policy). Backup MPC reduces errors by steering toward historical safe data, while constrained/penalized MPC shifts the error distribution to obstacle corners. Similarly, we conjecture that with the safety critic, querying at out-of-distribution inputs can minimize learned risk but still result in collision.}
    \label{fig: mod_demo}
\end{figure*}

\begin{figure}
    \centering
    \includegraphics[width=\linewidth]{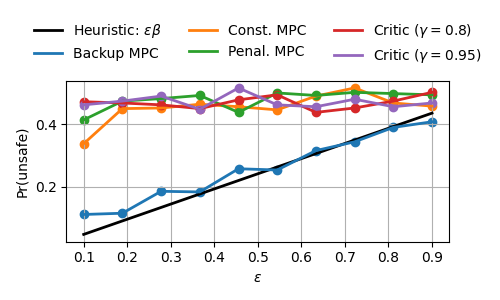}
    \caption{Comparing our policy modification approach (Backup MPC) against alternatives. For different $\epsilon$ and $20$ repetitions, we form the modified policy using each method and execute $25$ trajectories (totaling $500$ per $\epsilon$). By tracking to historical safe data upon alert our approach yields a heuristic safety guarantee, approximately achieving the error rate of $\epsilon \beta$ from Corollary~\ref{cor: warning_error_rate}. The alternative approaches fail to significantly reduce error rate with decreasing $\epsilon$, likely because of distribution shift over visited states and selected actions.}
    \label{fig: mod_systematic}
\end{figure}

Figure~\ref{fig: mod_demo} demonstrates our modified quadcopter MPC policy for $50$ test runs, where the base MPC policy is augmented with our backup safety mode. We construct the warning system offline using $\epsilon = 0.1$ and extract the associated polytopes defining $C(\epsilon)$ to constrain the backup policy $\pi_B$. Modifying the policy reduces the empirical collision rate from $52\%$ to $14\%$, and $\pi_B$ often successfully steers around obstacles. We compare this approach against three alternatives:
\begin{itemize}
    \item Constrained MPC: directly constrains the base policy $\pi$ to avoid $C(\epsilon)$, using a large slack penalty of $100$ to make avoidance act like a hard constraint (while still ensuring feasibility during SCP iteration).~\footnote{We add cost $c_{\delta} ||\delta||_1$ and relax constraints to $F_x(t) x \leq g_x(t) + \delta$.}
    \item Penalized MPC: reduces the slack penalty to $10$ to allow some cost-penalized entrance into $C(\epsilon)$.~\footnote{$p$-values above $\epsilon$ occur but not often, indicating non-negligible penalty.}
    \item Critic Recovery: offline pretraining of \cite{Thananjeyan2021}, which fits safety critic $\hat{Q}(x_t, u_t)$ to predict $\gamma$-discounted safety risk and switches to a recovery policy minimizing this risk when $\hat{Q} > \epsilon$.~\footnote{We minimize $\hat{Q}$ over a horizon via the cross entropy method (CEM).}
\end{itemize}
Constrained and penalized MPC still yield high collision rates ($32\%, 30\%$); the modified policies sometimes curve away from obstacles, but errors concentrate near obstacle corners. The critic recovery approach yields a $50\%$ error rate despite maintaining low predicted risk; we conjecture this occurs because the recovery policy queries $\hat{Q}$ at out-of-distribution inputs. In \cite{Thananjeyan2021}, we believe this problem is ameliorated because the initial offline data collects many more error states than in our limited data setting (whereas we use just $25$ unsafe states, they collect several hundred). Additionally, \cite{Thananjeyan2021} then iteratively improves the critic and policy through environment interaction. In our setting, we aim to provide a significant reduction in policy error rate one-shot from limited offline data.

Figure~\ref{fig: mod_systematic} provides a systematic comparison across $10$ $\epsilon$ values over $20$ data collection repetitions. For each repetition and $\epsilon$, we execute $25$ trajectories under each approach (totaling $500$ per $\epsilon$) and record the resulting policy error rate. While alternative methods show little to no improvement as $\epsilon$ decreases, our backup MPC consistently lowers the error rate. In fact, the modified policy approximately achieves the heuristic bound of $\epsilon \beta$ from Corollary~\ref{cor: warning_error_rate}. This heuristic assumes errors are predominantly caused by failing to alert, which is often reasonable since upon alert we switch to tracking historical safe data and should therefore fail rarely. Thus, we provide a heuristic safety guarantee that allows users to select $\epsilon$ to achieve a desired error rate $\eta$ for the modified policy: given estimated original error rate $\hat{\beta}$, take $\epsilon \leq \eta / \hat{\beta}$.

% Put hardware experiment earlier
\begin{figure}
     \centering
     \includegraphics[width=0.9\linewidth]{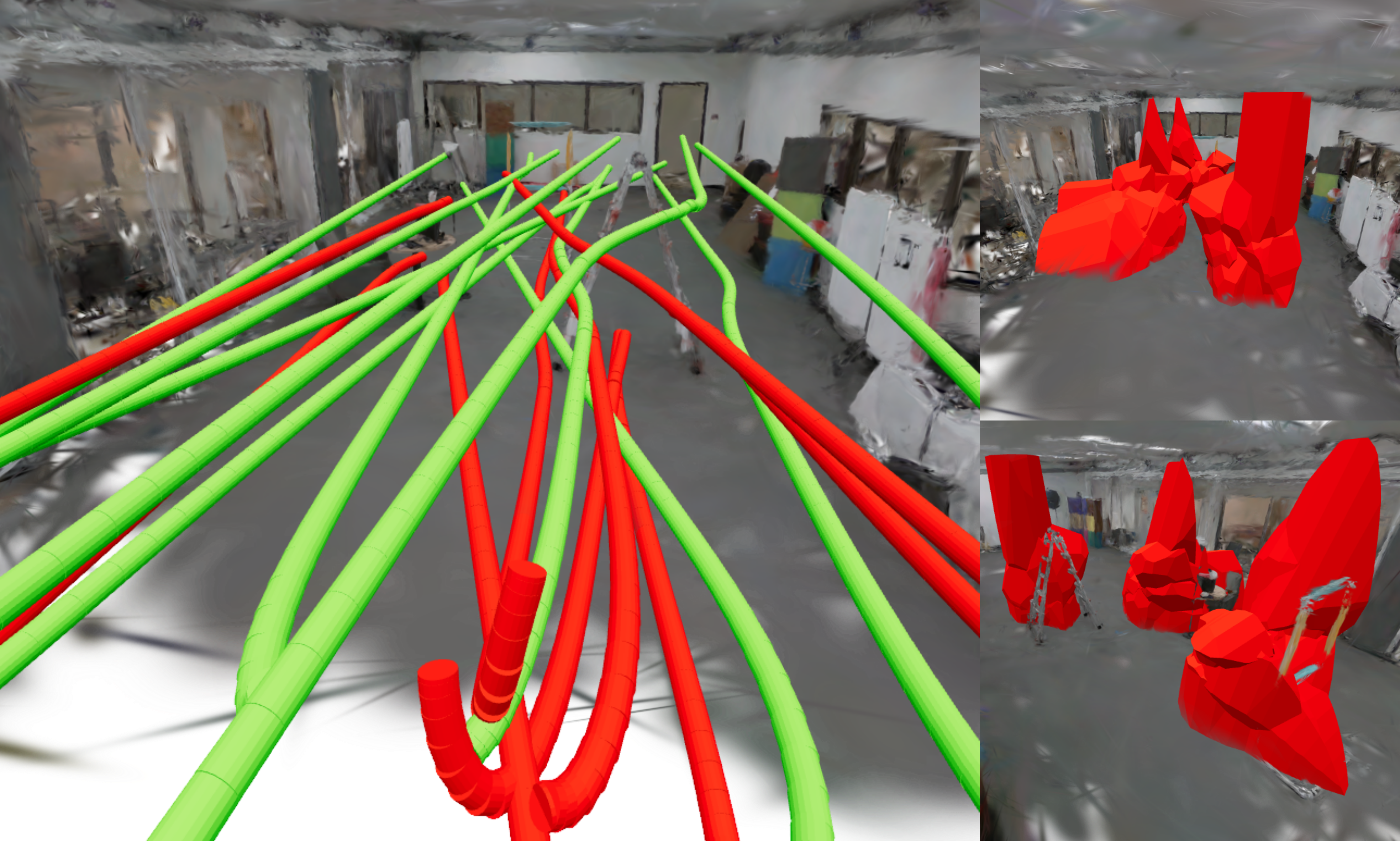}
    \caption{We modify a quadcopter MPC policy to more cautiously navigate around obstacles using human video feedback collected in a Gaussian Splat simulator. The labeled safe and unsafe trajectories are colored green and red respectively in the left subfigure. Using this data, we form the SUS region $C(\epsilon)$ in 9D state space, of which a 3D view is displayed on the right.}
    \label{fig: hardware_results}
\end{figure}

% \begin{figure}
%      \centering
%      % \subfloat[]{\includesvg[width=0.9\linewidth]{Figures/modify_sim_figs/backup.svg}\label{fig: backup}}
%      \subfloat[With proposed backup MPC.]{\includegraphics[width=0.9\linewidth]{Figures/modify_sim_figs/backup.png}\label{fig: sim_backup}}
%      \hfill
%      % \subfloat[]{\includesvg[width=0.9\linewidth]{Figures/modify_sim_figs/no_track.svg}\label{fig: no_track}}
%      \subfloat[Naively constrained nominal MPC.]{\includegraphics[width=0.9\linewidth]{Figures/modify_sim_figs/no_track.png}\label{fig: no_track}}
%     \caption{Modifying a nominal MPC policy to avoid $C(\epsilon)$. We specify $\epsilon = 0.1$ and use $N = 25$ errors with the unsafe-safe nearest neighbor conformal prediction procedure to find $C(\epsilon)$. Fig.~\ref{fig: sim_backup} shows $50$ trajectories executed with our proposed backup MPC that steers to historical safe data while avoiding $C(\epsilon)$. The empirical collision rate is reduced from $52\%$ to $6\%$. Fig.~\ref{fig: no_track} shows the original policy with constraints added to avoid $C(\epsilon)$. This induces a distribution shift and violates the calibration of $C(\epsilon)$, leading to high collision rates ($36\%$ in this case).}
%     \label{fig: mod_demo}
% \end{figure}

\subsection{Hardware Experiments}

For our hardware experiments, we start with a quadcopter MPC policy that has a minimal level of obstacle avoidance and use human feedback to modify the initial policy to navigate more cautiously around obstacles.~\footnote{Starting with some avoidance was necessary as an obstacle-unaware policy was found to almost always collide, yielding no backup trajectories.} To embed obstacle avoidance directly into the starting MPC policy, we at runtime approximate the local point cloud, derived from a Gaussian Splat simulator of the test environment, as a union of balls and impose these as MPC constraints.~\footnote{\cite{chen2024splatnavsaferealtimerobot} explores collision-free navigation with Gaussian Splats.} To perform this point cloud approximation, we treat the 3D points in the local point cloud as unsafe samples, and (as we lack complementary safe points) perform unsafe-only nearest neighbor conformal prediction (further detail in the Appendix).

We then refine this initial obstacle-aware policy using human-labeled data, building the SUS region via the unsafe-safe procedure in the full quadcopter state space. Using the Gaussian Splat of the test environment, the user watches simulated flight trajectories and terminates when the quadcopter is perceived to become unsafe. Acting as human labeler, one author reviewed $100$ quadcopter trajectories and flagged $52$ of them as unsafe. These failures were caused by the quadcopter (i) passing too close above the table or trying to go underneath it, (ii) going through the ladder instead of around it, or (iii) going too near/under the gate, and reflect the subjectivity associated with human-determined safety. Figure~\ref{fig: hardware_results} (left) shows $20$ of these labeled trajectories overlaid on a Gaussian Splat rendering of the test environment, marking safe trajectories in green and unsafe in red. When fitting $C(\epsilon)$, using the unsafe-safe score (Eq.~\ref{eq: two_sample_special}), we selected $\epsilon = 0.096$ so the heuristic error rate after policy modification would be $\eta = 0.05$. We visualize a 3-dimensional view of $C(\epsilon)$ (which exists in 9 dimensions) overlaid on the Gaussian Splat in Figure~\ref{fig: hardware_results} (right). $C(\epsilon)$ covers the main obstacles (table, lamp, ladder, and gate) flagged by the human. Furthermore, it reflects the preemptive human labeling as the polyhedra protrude outwards from the front of obstacles.

% \begin{figure}
%      \centering
%      \subfloat[]{\includegraphics[width=0.7\linewidth]{Figures/modify_hardware_figs/binary_labels_cropped.png}\label{fig: hardware_train}}
%      \hfill
%      \subfloat[]{\includesvg[width=0.3\linewidth]{Figures/modify_hardware_figs/poly_merged_vert.svg}\label{fig: hardware_geom}}
%     \caption{We modify a quadcopter MPC policy to more cautiously navigate around obstacles using human video feedback collected in a Gaussian Splat simulator. The labeled safe and unsafe trajectories are colored green and red respectively in Fig.~\ref{fig: hardware_train}. Using this data, we form the SUS region $C(\epsilon)$ in 9-dimensional state space, of which a 3-dimensional view is shown in Fig.~\ref{fig: hardware_geom}.}
%     \label{fig: hardware_results}
% \end{figure}

Using the approach described in Section~\ref{sec: policy_mod}, we modify the original MPC policy and perform $30$ flight tests in hardware, performing $5$ trials for $6$ start-goal configurations. We used an off-board desktop computer and streamed the resulting action to a low-level PX4 controller onboard the quadcopter using ROS2, and used an EKF on motion capture pose estimates for state information. We manually piloted to the starting location, allowed MPC to run for $25$ seconds, then manually landed.

We tested from two start locations (roughly facing the table or the ladder) and targeting three different goals across the room (left, center, or right side). Figure~\ref{fig: hardware_comparison} compares these flight tests (in green) to simulations using the modified (in blue) or original (in red) policy. The trajectories using the modified policy avoid some of the original unsafe behavior; they go higher above the table (see chair side left goal), do not get stuck near the gate (see chair side center goal), and do not go through the ladder (see ladder side center goal). 

For each configuration, Figure~\ref{fig: hardware_comparison} plots the associated $p$-value over time (see subsection~\ref{subsec: p_val}). In simulation, the modified policy (in blue) keeps the $p$-values below the $\epsilon$ threshold while the original policy (in red) exceeds this threshold when nearing the ladder or gate. When the red and blue $p$-value curves fully overlap this reflects that the original policy was deemed safe enough that the backup safety mode was never triggered. The green $p$-value curves associated with the hardware tests, $5$ for each configuration, generally remain below the $\epsilon$ cutoff. However, they occasionally violate (unlike the modified policy in simulation) which may be attributed to sim-to-real gap e.g., deviation from the modeled dynamics due to low-level flight control, state estimation error etc. For instance, there is an outlier in ladder side right goal where one of the $5$ hardware tests features a $p$-value spike. Nonetheless, the associated trajectory avoids going under the ladder; it starts moving to the ladder center but then shifts right.

% ~\footnote{In point of fact, for one configuration (fixed start and goal) simulated trajectories are deterministic but the observed hardware tests differ run to run.}

% \begin{figure}
%     \centering
%     \subfloat[]{\includesvg[width=0.65\linewidth]{Figures/modify_hardware_figs/chair_fig.svg}\label{fig: chair_fig}}
%     \hfill
%     \subfloat[]{\includesvg[width=0.65\linewidth]{Figures/modify_hardware_figs/ladder_fig.svg}\label{fig: ladder_fig}}
%     \caption{Hardware flight results (in green) using an MPC policy modified via human feedback. We compare to simulated trajectories of the modified (in blue) or original (in red) policy. We performed $5$ tests from $6$ configurations (2 starting locations and 3 goals). The modified policy avoids unsafe behavior flagged by the human (e.g., going under the ladder) and generally keeps $p$-values below the user-specified $\epsilon$ threshold.}
%     \label{fig: hardware_comparison}
% \end{figure}

\begin{figure}
    \centering
    \includegraphics[width=0.68\linewidth]{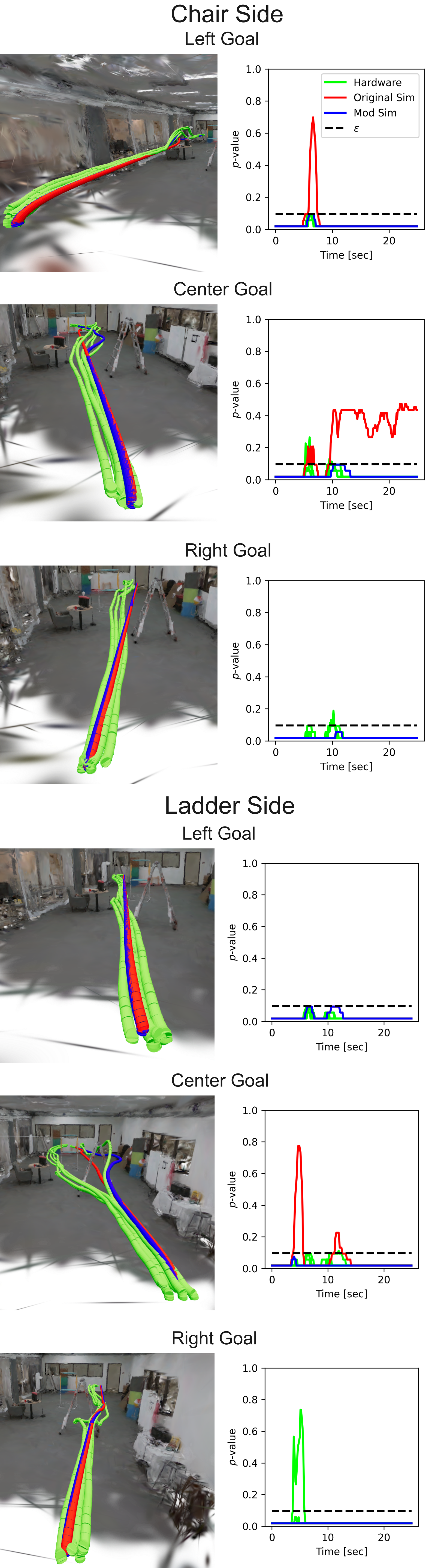}
    \caption{Hardware flight results (in green) using an MPC policy modified via human feedback. We compare to simulated trajectories of the modified (in blue) or original (in red) policy. We performed $5$ tests from $6$ configurations (2 starting locations and 3 goals). The modified policy avoids unsafe behavior flagged by the human (e.g., going under the ladder) and generally keeps $p$-values below the user-specified $\epsilon$ threshold.}
    \label{fig: hardware_comparison}
\end{figure}

% \begin{figure}
%      \centering
%      \subfloat[]{\includegraphics[width=0.9\linewidth]{Figures/modify_hardware_figs/comparison.png}\label{fig: traj_comparison}}
%      \hfill
%      \subfloat[]{\includesvg[width=\linewidth]{Figures/modify_hardware_figs/p_val_comp.svg}}\label{fig: p_val_comp}
%     \caption{Comparison of hardware flight tests (shown in lighter color) using an MPC policy modified via human feedback to simulated trajectories (shown in darker color) obtained using the original MPC policy. From Fig.~\ref{fig: traj_comparison}, we observe that the modified policy avoids some of the original unsafe behavior flagged by the human; it acts more cautiously, not going under the ladder and giving it wider berth. Fig.~\ref{fig: p_val_comp} shows that the modified MPC policy indeed keeps the $p$-value below the user-specified $\epsilon$ threshold while the original policy does not.}
%     \label{fig: hardware_comparison}
% \end{figure}
\section{Conclusion}
\label{sec: conclusion}

Defining and ensuring safety for a robot policy can be challenging due to environmental uncertainty, policy complexity, or human subjectivity. In this work, we developed an approach for learning about robot safety from human-supplied labels. The human observes a handful of policy executions and terminates if they believe the robot reaches an unsafe state. Using this data, we fit a nearest neighbor classifier and calibrate it using conformal prediction to correctly flag a user-defined fraction of future errors. We derive novel theory showing we can calibrate without reserving unsafe states (of which there may be few) and provide an equivalent geometric description of the suspected unsafe sublevel (SUS) region. Using this region, we provide a warning system to mimic the safety preferences of the human labeler, up to a user-defined miss rate. The associated scalar $p$-value can serve as an interpretable safety score for runtime monitoring. We demonstrate the warning system for simulated quadcopter MPC and visuomotor policies. Lastly, we use the SUS region for policy modification by adding a backup safety mode which triggers upon warning system alert. We test this approach in quadcopter hardware experiments.

\section{Limitations}
\label{sec: limitations}

There are several avenues to extend our work and overcome some of its limitations. Our work assumed a static unsafe set but extending to open-world dynamic scenarios is an important next step. This may be achieved by using our method to learn about safety directly from visual data \cite{dino, dinov2}, as initially explored in Section~\ref{sec: warning_exp}. More generally, reasoning about safety in the robot’s body frame can enable transfer across environments (with novel obstacle geometries). For example, instead of learning to avoid a specific table, the warning system could ingest embedded point clouds in quadcopter body frame and learn that flying under any table is unsafe. Another extension would be to use nearest neighbor conformal prediction for policy modification beyond MPC. For instance, the calibrated $p$-value could be used as a penalty in reinforcement learning. Lastly, it would be valuable to re-calibrate the SUS region for new policies or conditions without re-collecting data, possibly by weighting samples in conformal calibration \cite{tibshirani2020conformalpredictioncovariateshift}. % Weighting could also be used to build the SUS region more efficiently by searching for risky initial conditions \cite{moss2024bayesian}.

% \section*{Acknowledgments}
% The authors would like to acknowledge Keiko Nagami for help with the quadcopter dynamics and control formulation and Tim Chen for help using the Gaussian Splat.
\appendix

\section*{Connection with Gaussian Kernel Density Estimation and Likelihood Ratio Testing}
\label{subsec: kde}

$k$-nearest neighbor distance has long been used for nonparametric density estimation \cite{Akaike1954AnAT, Biau2015}. Here we make a separate connection between Gaussian kernel density estimation (KDE) and nearest neighbor distance that motivates Euclidean distance in the unsafe-only score and the difference of squared Euclidean distances for the unsafe-safe score (Eq.~\ref{eq: two_sample_special}).

Given $\{x_i\}_{i=1}^N$ IID from $F$, we seek $C(\epsilon)$ containing $1-\epsilon$ of future samples. Given density $\rho_F$, we could set $C = \{x \mid \rho_F(x) \geq r\}$ and find $r$ so that $\int_C \rho_F(x) = 1-\epsilon$. $\rho_F$ is unknown, so we estimate $\hat{\rho}_F$ from samples e.g., via Gaussian KDE. Building $C(\epsilon)$ with Euclidean nearest neighbor conformal score can be viewed as approximately conformalizing Gaussian KDE for a small bandwidth $h$~\cite{Lei2013DistributionFreePS}. In Gaussian KDE, $\rho_F$ over $\mathbb{R}^p$ is approximated by replacing the Dirac with a Gaussian $N(0, hI)$ density (converging as $h$ shrinks):
\begin{subequations}
\begin{gather}
\rho_F(x) = \int \rho_F(x') \delta(x'-x) dx' \approx \\
\int \rho_F(x') \frac{1}{(2 \pi h)^{(p/2)}} \exp(-\frac{1}{2h} ||x - x'||^2) dx'.
\end{gather}
\end{subequations}
Equating to an expectation and applying Monte Carlo yields
\begin{subequations}
\begin{gather}
\mathbb{E}_{x' \sim F}\big[\frac{1}{(2 \pi h)^{(p/2)}} \exp(-\frac{1}{2h} ||x - x'||^2)\big] \approx \\
1/N \sum_{i=1}^N \frac{1}{(2 \pi h)^{(p/2)}} \exp(-\frac{1}{2h} ||x - x_i||^2) = \hat{\rho}_F(x).
\end{gather}
\end{subequations}

To define $C(\epsilon)$, we replace $\rho_F(x) \geq r$ by $\hat{\rho}_F(x) \geq r$. Applying a logarithm (which is monotonic), dropping data-independent constants, and performing positive re-scaling yields equivalent characterization \footnote{We repeatedly overload $r$'s value as it is calibrated for $1-\epsilon$ coverage.}
\begin{equation}
LSE(-\gamma ||x - x_1||^2, ..., -\gamma ||x - x_N||^2) \geq r
\end{equation}
where $LSE(a_1, ..., a_n)$ denotes log-sum-exp and $\gamma = \frac{1}{2h}$. As 
\begin{subequations}
\begin{gather}
LSE(-\gamma ||x - x_1||^2, ..., -\gamma ||x - x_N||^2) \in \\ \nonumber [\max_i -\gamma ||x - x_i||^2,\ \max_i -\gamma ||x - x_i||^2 + \log(N)]
\implies \\
LSE(-\gamma ||x - x_1||^2, ..., -\gamma ||x - x_N||^2) \approx -\gamma \min_i ||x - x_i||^2
\end{gather}
\end{subequations}
for small $h$ (large $\gamma$). Dropping $-\gamma$ yields equivalent test $\min_i ||x - x_i||^2 \leq r$. 

In the unsafe-safe case, we have $\{y_i\}_{i=1}^M$ from $G$ and want to determine if point $x$ is sampled from $F$ or $G$. Given $\rho_F, \rho_G$, a classic approach is the likelihood ratio test~\cite{mood1974introduction}, forming score $s(x) = \log(\rho_F(x)) - \log(\rho_G(x)) \geq r$ and declaring $x$ from $F$ if $s(x) \geq r$ and else $G$. We apply our previous approximation twice to the test $\log(\hat{\rho}_F(x)) - \log(\hat{\rho}_G(x)) \geq r$, concluding $x$ drawn from $F$ when $\min_i ||x - x_i||^2 - \min_i ||x - y_i||^2 \leq r$, for appropriately calibrated $r$. In summary, Eq.~\ref{eq: two_sample_special} can be viewed as an approximate likelihood ratio test using small-bandwidth Gaussian KDE. Since the likelihood ratio test is uniformly most powerful~\cite{mood1974introduction}, we can hope this score will perform well.

% In this context, we view the binary classification problem as binary hypothesis testing. The null hypothesis H0 is that $x$ is drawn from $F$ and the alternative H1 is that $x$ is drawn from $G$. The cutoff threshold $r$ is chosen to guarantee the Type 1 error rate $\Pr(\mbox{reject H0} \mid \mbox{H0 holds}) \leq \epsilon$. This is precisely what we have called the miss rate i.e., $\Pr(\mbox{classify x not drawn from F} \mid \mbox{x drawn from F}) \leq \epsilon$.

% Could have a two-sample experiment with two Gaussian distributions and compare this approach to optimal Neyman-Pearson test and the one-sample CP. Also show the geometry of the rejection region for all three approaches.

\section*{Point Cloud Obstacle Avoidance}
For initial obstacle avoidance in our hardware experiments we used our unsafe-only nearest neighbor method to approximate the point cloud (the Gaussian Splat ellipsoid centers). We extract all $N$ nearby points (there may be several thousand) to the current quadcopter position. To contain $1-\overline{\epsilon} = 0.95$ of these $N$ points, we randomly subselect $n = 200$ points and build balls around them. $C(\epsilon)$ will contain the $n$ subselected points and in expectation $(1-\epsilon) (N - n)$ of the remaining points. Thus, we choose $\epsilon$ such that $n + (1-\epsilon) (N - n) = (1-\overline{\epsilon}) N$ and set radius $r$ using Theorem~\ref{thm: main_cp}. To avoid each ball in MPC, we linearize $h_i(x) = (x - x_i)^T (x - x_i) - r > 0$ about previous guess $\overline{x}_t$. Figure~\ref{fig: point_cloud} visualizes our $C(\epsilon)$ point cloud approximation for a quadcopter position near the ladder.

% To promote trajectory smoothness, we tightened saturation limits: $8 \leq F \leq 11$, $-0.3 \leq \omega_j \leq 0.3 \ j=x,y,z$, added cost $\sum_{t=0}^{T-2} ||u_{t+1} - u_t||_S^2$, $S = diag(0,100,100,100)$, and used $q_z = 100, q_e = 5$. To generate trajectories for human labeling, we selected random starting and goal locations on opposite sides of the room, including randomizing the heights and starting yaw.

\begin{figure}
    \centering
    \includegraphics[width=0.6\linewidth]{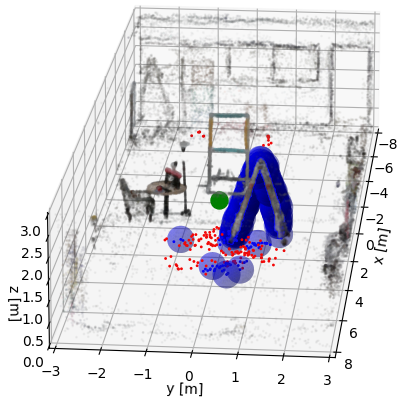}
    \caption{Visualization of the point cloud obtained from a Gaussian Splat of the test environment. For a quadcopter position (shown in green) near the ladder, we show (in blue) the union of balls $C(\epsilon)$ approximation to the nearby point cloud. We color nearby points blue if contained in $C(\epsilon)$ and red otherwise.}
    \label{fig: point_cloud}
\end{figure}

\section*{Proofs}

\begin{lemma}[Conformal Guarantee with Ties]
\label{lemma: cp_ties}
Let $S = \{s_i\}_{i=1}^{N+1}$ be exchangeable scores. Assume almost surely $Z_2 \coloneqq \sum_{i=1}^N \1{s_i = s_{N+1}} \leq T$. Then,
\begin{equation}
    k/(N+1) \leq \Pr[s_{N+1} \leq s_{(k)}] \leq (k + T)/(N+1).
\end{equation}
\end{lemma}
where $s_{(k)}$ is the $k$'th order statistic of $\{s_1, ..., s_N\}$.
\begin{proof}
Imagine an ascending sort of $S$, breaking ties arbitrarily. Let the rank of $s_i$ denote its index in the sorted version of $S$. By exchangeability, $\Pr(s_{N+1} \mbox{ ranked } k) = 1/(N+1)$ for each $k \in \{1, ..., N+1\}$ implying
\begin{equation}
\label{eq: informal_rank}
\Pr(s_{N+1} \mbox{ ranked } \leq k) = k/(N+1).
\end{equation}

Defining $Z_1 = \sum_{i=1}^N \1{s_i < s_{N+1}}$, the rank of $s_{N+1}$ is $Z_1 + 1 + U$: $Z_1 + 1$ as $s_{N+1}$ must appear after $s_i$ strictly less than it and $U$ is a uniform random variable on $\{0, 1, ..., Z_2\}$ to model random tie splitting. With this notation, Eq.~\ref{eq: informal_rank} becomes
\begin{equation}
\label{eq: formal_rank}
\Pr(Z_1 + 1 + U \leq k) = k/(N+1).
\end{equation}

Since $s_{N+1} \leq s_{(k)}$ when at most $k-1$ of $s_1, ..., s_N$ violate $s_{N+1} \leq s_i$ (i.e., $s_i < s_{N+1}$ holds),
\begin{align}
s_{N+1} \leq s_{(k)} \iff Z_1 = \sum_{i=1}^N \1{s_i < s_{N+1}} \leq k-1
\end{align}
\begin{align}
\label{eq: z1_to_s}
\implies \Pr(s_{N+1} \leq s_{(k)}) = \Pr(Z_1 + 1 \leq k).
\end{align}

Since $0 \leq U \leq Z_2 \leq T$,
\begin{align}
    \Pr(Z_1 + 1 \leq k) \geq \Pr(Z_1 + 1 + U \leq k) \\ \nonumber
    \geq \Pr(Z_1 + 1 + T \leq k).
\end{align}
By Eq.~\ref{eq: formal_rank} and the above,
\begin{align}
    \label{eq: z1_sandwich}
    \Pr(Z_1+1 \leq k) \geq k/(N+1) \geq \Pr(Z_1+1 \leq k-T).
\end{align}
Eq.~\ref{eq: z1_to_s} and the left side of Eq.~\ref{eq: z1_sandwich} yield $\Pr(s_{N+1} \leq s_{(k)}) \geq k/(N+1)$. Relabeling $k$ to $k+T$ in the right side of Eq.~\ref{eq: z1_sandwich} and applying Eq.~\ref{eq: z1_to_s} yields
\begin{equation}
(k+T)/(N+1) \geq \Pr(Z_1+1 \leq k) = \Pr(s_{N+1} \leq s_{(k)}).
\end{equation}
\end{proof}

% \begin{lemma}
% \label{lemma: order_stat_comp}
% Let $\{s_1, ..., s_N\}$ and $\{\alpha_1, ..., \alpha_N\}$ satisfy $s_i \leq \alpha_i \ \forall i=1,...,N$. Then, $\forall k = 1,...,N$, $s_{(k)} \leq \alpha_{(k)}$.
% \end{lemma}
% \begin{proof}
% Sort the $\alpha_i$ in ascending order yielding sorted indices $(i_1, ..., i_N)$ i.e., $\alpha_{(k)} = \alpha_{i_k}$. Since $s_i \leq \alpha_i$ for each $i$, $s_{i_j} \leq \alpha_{i_j}$ and $\alpha_{i_j} \leq \alpha_{i_k}$ for $j \leq k$. Hence $s_{i_1}, ..., s_{i_k}$ are each $\leq \alpha_{i_k}$ so $s_{(k)} \leq \alpha_{i_k} = \alpha_{(k)}$.
% \end{proof}

\begin{lemma}
\label{lemma: order_stat_comp_2}
Consider sets $\{s_1, ..., s_N\}$ and $\{\alpha_1, ..., \alpha_N\}$. Suppose $\forall i: s_{N+1} \leq s_i \implies s_{N+1} \leq \alpha_i$. Then, $\forall k: s_{N+1} \leq s_{(k)} \implies s_{N+1} \leq \alpha_{(k)}$.
\end{lemma}
\begin{proof}
Assuming $s_{N+1} \leq s_{(k)}$, $< k$ $\{s_i\}_{i=1}^N$ satisfy $s_i < s_{N+1}$, so $>N-k$ satisfy $s_{N+1} \leq s_i$:
\begin{equation}
\sum_{i=1}^N \1{s_{N+1} \leq s_i} > N-k.
\end{equation}
Since $s_{N+1} \leq s_i \implies s_{N+1} \leq \alpha_i$, 
\begin{equation}
\sum_{i=1}^N \1{s_{N+1} \leq \alpha_i} > N-k \implies \sum_{i=1}^N \1{\alpha_i < s_{N+1}} < k
\end{equation}
and so $s_{N+1} \leq \alpha_{(k)}$.
\end{proof}

\subsection{Proof of Theorem \ref{thm: main_cp}}
\begin{proof}
Let $D_{-i} = \{x_1, x_2, ..., x_{i-1}, x, x_{i+1}, ..., x_N\}$ formed by swapping $x \coloneqq x_{N+1}$ into the $i$'th position of $D$. Consider the $i$-th score generated in full conformal prediction with the nearest neighbor score function (Eq.~\ref{eq: score func}):
\begin{subequations}
\label{eq: s_to_alpha_bound}
\begin{align}
    s_i^x = s(D_{-i}; x_i) = \min_{x' \in D_{-i}} d(x', x_i) = \\ \min\{\alpha_i, d(x, x_i)\} \leq \alpha_i
\end{align}
\end{subequations}
by splitting to within the data $\alpha_i = \min_{x' \in D, x' \neq x_i} d(x', x_i)$ and with the test point $d(x, x_i)$. Since $s_i^x \leq \alpha_i \ \forall i$,
\begin{equation}
s_{N+1}^x \leq s_i^x \implies s_{N+1}^x \leq \alpha_i
\end{equation} and so by Lemma~\ref{lemma: order_stat_comp_2},
\begin{equation}
\label{eq: s_to_alpha}
s_{N+1}^x \leq s_{(k)}^x \implies s_{N+1}^x \leq \alpha_{(k)}.
\end{equation}
Hence, also applying Lemma~\ref{lemma: cp_ties}, 
\begin{equation}
    \Pr(s_{N+1}^x \leq \alpha_{(k)}) \geq \Pr(s_{N+1}^x \leq s_{(k)}^x) \geq k/(N+1).
\end{equation}
$k = k(\epsilon)$ from Eq.~\ref{eq: k_epsilon} satisfies $k/(N+1) \geq 1-\epsilon$ so as $s_{N+1}^x = S(D; x)$, $\alpha_{(k)} = r$,
\begin{equation}
\Pr(s(D; x) \leq r) \geq 1-\epsilon.
\end{equation}
\end{proof}

\subsection{Proof of Theorem \ref{thm: union_cp}}
\begin{proof}
The condition $s(D;x) \leq r$ defining $C(\epsilon)$ means
\begin{equation}
    \min_{x' \in D} d(x', x) \leq r \iff \exists i \in \{1, ..., N\}: d(x_i, x) \leq r
\end{equation}
Grouping by $x_i$, we obtain $C_i$ by requiring $d(x_i, x) \leq r$, yielding $C(\epsilon)$ as the union of $C_i$ sets.
\end{proof}

\subsection{Proof of Theorem \ref{thm: overcoverage_sym}}
\begin{proof}
With no pairwise distance ties, almost surely (with probability 1) each $x_i$ has a unique nearest neighbor and $F$ has zero probability of repeatedly sampling a point.

% Firstly, note that each $x_i$ has a unique nearest neighbor with probability 1 (almost surely). Suppose not, i.e., that $min_{x' \in D_{-i}} d(x', x_i)$ has two minimizers $x_j, x_k \in D_{-i}$. Then $d(x_j, x_i) = d(x_k, x_i)$, and this event occurs with probability 0 by assumption. 

% Secondly, note that to satisfy the theorem assumptions, $F$ must have zero probability of repeatedly sampling a given state (i.e, be continuous). Otherwise, there is nonzero probability of sampling a single state three times $x_1 = x_2 = x_3$ and so $d(x_1, x_2) = d(x_1, x_3)$ violating the no pairwise ties assumption.

We'll show $Z_2 \leq 1$ almost surely. Suppose $Z_2 > 1$. Then $\exists j \neq i \in \{1, ..., N\}$ such that $s_i^x = s_j^x = s_{N+1}^x$. $s_i^x$, $s_j^x$, $s_{N+1}^x$ almost surely have unique minimizers $x_i^*, x_j^*, x_{N+1}^*$ satisfying
\begin{equation}
\label{eq: common_dist}
d(x_i^*, x_i) = d(x_j^*, x_j) = d(x_{N+1}^*, x_{N+1}).
\end{equation}
$d(x_i^*, x_i) = d(x_j^*, x_j)$ occurs (with nonzero probability) when $x_i^* = x_j, x_j^* = x_i$ (i.e., $x_i, x_j$ are each others nearest neighbors). Similarly, $d(x_j^*, x_j) = d(x_{N+1}^*, x_{N+1})$ when $x_j^* = x_{N+1}$, $x_{N+1}^* = x_j$. Therefore, \begin{equation}
    x_j^* = x_i, x_j^* = x_{N+1} \implies x_i = x_{N+1}
\end{equation}
which occurs with probability $0$. So $Z_2 > 1$ requires conditions occuring with probability zero. Lemma~\ref{lemma: cp_ties} with $T = 1$ implies 
\begin{equation}
\label{eq: s_upper_bound}
\Pr(s_{N+1}^x \leq s_{(k)}^x) \leq (k+1)/(N+1).
\end{equation}

Since $d$ is symmetric,
\begin{equation}
    s_{N+1}^x = \min_j d(x_j, x) \leq d(x_i, x) = d(x, x_i).
\end{equation}
Hence, if $s_{N+1}^x \leq \alpha_i$, then $s_{N+1}^x \leq \min \{\alpha_i, d(x, x_i)\} = s_i^x$ (see Theorem~\ref{thm: main_cp}). Conclude $s_{N+1}^x \leq \alpha_i \implies s_{N+1}^x \leq s_i^x$. The converse also holds since $s_i^x \leq \alpha_i$ (see Theorem~\ref{thm: main_cp}) so
\begin{equation}
s_{N+1}^x \leq s_i^x \iff s_{N+1}^x \leq \alpha_i.
\end{equation}
Applying Lemma~\ref{lemma: order_stat_comp_2} twice (for each implication direction),
\begin{equation}
s_{N+1}^x \leq s_{(k)}^x \iff s_{N+1}^x \leq \alpha_{(k)}
\end{equation}
and so by Eq.~\ref{eq: s_upper_bound}
\begin{equation}
\label{eq: alpha_bound}
\Pr(s_{N+1}^x \leq \alpha_{(k)}) = \Pr(s_{N+1}^x \leq s_{(k)}^x) \leq (k+1)/(N+1).
\end{equation}

$k(\epsilon) = \lceil(1-\epsilon)(N+1)\rceil \leq (1-\epsilon)(N+1) + 1$ implies $\frac{k+1}{N+1} \leq (1-\epsilon) + 2/(N+1)$. Eq.~\ref{eq: alpha_bound} with $s_{N+1}^x = S(D; x)$, $k = k(\epsilon), \alpha_{(k)} = r$ yields
\begin{equation}
\Pr(s(D; x) \leq r) \leq 1-\epsilon + 2/(N+1).
\end{equation}
\end{proof}

\begin{remark}
This bound can be tight. As $\epsilon \rightarrow 1$, $k(\epsilon) = 1$, so 
\begin{equation}
    r = \alpha_{(1)} = \min_{x_i \neq x_j \in D} d(x_i, x_j) = \alpha_{(2)}
\end{equation}
as the minimizing $(x_i, x_j)$ each have an $\alpha_i, \alpha_j$ term. As $r = \alpha_{(2)}$, using $k = 2$ shows $C(\epsilon)$ provides $\geq 2/(N+1)$ coverage.
\end{remark}

\subsection{Proof of Theorem \ref{thm: overcoverage_asym}}

\begin{proof}
Because $\{d(x_i,x_j)\}_{i \neq j}$ have no ties, $Z_2 = 0$ almost surely. Applying Lemma~\ref{lemma: cp_ties} with $T = 0$ yields \begin{equation}
\label{eq: s_bound_asym}
\Pr(s_{N+1}^x \leq s_{(k)}^x) = \frac{k}{N+1}.
\end{equation}

% Consider the event that $Z_2 > 0$. This requires that $\exists i \in \{1, ..., N\}$ such that $s_i^x = s_{N+1}^x$. Let $x_i^*, x_{N+1}^*$ denote the nearest neighbors of $x_i$ and $x_{N+1}$ respectively. So $d(x_i^*, x_i) = d(x_{N+1}^*, x_{N+1})$ violating the no repeated pairwise distances assumption which holds with probability 1. Hence, $Z_2 > 0$ with probability 0 so $Z_2 = 0$ with probability 1.

To bound the coverage we split into two cases:

\begin{align}
    \Pr(s_{N+1}^x \leq \alpha_{(k)}) = \Pr(s_{N+1}^x \leq \alpha_{(k)}, s_{(k)}^x \geq \alpha_{(k)}) + \\ \nonumber
    \Pr(s_{N+1}^x \leq \alpha_{(k)}, s_{(k)}^x < \alpha_{(k)})
\end{align}

We bound the first term using Eq.~\ref{eq: s_bound_asym}:
\begin{align}
\Pr(s_{N+1}^x \leq \alpha_{(k)}, s_{(k)}^x \geq \alpha_{(k)}) \leq \\ \nonumber
\Pr(s_{N+1}^x \leq s_{(k)}^x, s_{(k)}^x \geq \alpha_{(k)}) \leq \\ \nonumber
\Pr(s_{N+1}^x \leq s_{(k)}^x) = \frac{k}{N+1}
\end{align}

We bound the second term by
\begin{equation}
\Pr(s_{N+1}^x \leq \alpha_{(k)}, s_{(k)}^x < \alpha_{(k)}) \leq \Pr(s_{(k)}^x < \alpha_{(k)}).
\end{equation}

Imagine an ascending sort of $\{\alpha_i\}_{i=1}^N$ yielding sorted indices $(i_1, i_2, ..., i_N)$ where $\alpha_{i_k} = \alpha_{(k)}$. Since $s_i^x \leq \alpha_i \ \forall i$ (Eq.~\ref{eq: s_to_alpha_bound}), $s_{i_j}^x \leq \alpha_{i_j}$ for $j < k$. Without distance ties, $\alpha_{i_j} < \alpha_{i_k}$ for $j < k$, so $s_{i_j}^x < \alpha_{i_k} = \alpha_{(k)}$ almost surely. Almost surely,
\begin{subequations}
\label{eq: elim_min_equiv}
\begin{align}
    s_{(k)}^x < \alpha_{(k)} \iff \exists i_j, j \geq k: s_{i_j}^x < \alpha_{i_k} \iff \\
    \min\{s_{i_j}^x\}_{j=k:N} < \min \{\alpha_{i_j}\}_{j=k:N}.
\end{align}
\end{subequations}
Consider the distances used to form $s_{i_j}^x$ and $\alpha_{i_j}$ for $j = k, ..., N$. $\alpha_{i_j}$ takes the minimum over $N-1$ terms $D_j^\alpha := \{d(x_i, x_{i_j})\}_{i=1:N, i \neq i_j}$ while $s_{i_j}^x$ takes the minimum over $N$ terms $D_j^s = D_j^\alpha \cup \{d(x_{N+1}, x_{i_j})\}$. Thus, $\min\{s_{i_j}^x\}_{j=k:N}$ is the minimum among $(N - k + 1)N$ distances $D^s := \cup_{j=k:N} D_j^s$ while $\min\{\alpha_{i_j}\}_{j=k:N}$ is the minimum among subset $D^\alpha := \cup_{j=k:N} D_j^\alpha \subset D^s$ of size $(N - k + 1) (N - 1)$. When $\min\{s_{i_j}^x\}_{j=k:N} \in D^s - D^\alpha$, $\min\{s_{i_j}^x\}_{j=k:N} < \min\{\alpha_{i_j}\}_{j=k:N}$ (else they are equal). Each distance is equally likely to be $\min\{s_{i_j}^x\}_{j=k:N}$, so counting yields
\begin{subequations}
\begin{align}
\Pr(\min \{s_{i_j}^x\}_{j=k:N} < \min \{\alpha_{i_j}\}_{j=k:N}) = \\  \Pr(\min \{s_{i_j}^x\}_{j=k:N} \in D^s - D^\alpha) = \\ \frac{(N - k + 1)N - (N - k + 1)(N-1)}{(N - k + 1)N} = 1/N.
\end{align}
\end{subequations}

Using the equivalence in Eq.~\ref{eq: elim_min_equiv}, conclude
\begin{equation}
\Pr(s_{(k)}^x < \alpha_{(k)}) = 1/N.
\end{equation}

Combining with our bound on the first term,
\begin{equation}
\label{eq: alpha_bound_asym}
    \Pr(s_{N+1}^x \leq \alpha_{(k)}^x) \leq k/(N+1) + 1/N
\end{equation}

Using $k(\epsilon)$ as in Eq.~\ref{eq: k_epsilon} implies $\frac{k}{N+1} \leq 1-\epsilon + 1/(N+1)$. Eq.~\ref{eq: alpha_bound_asym} with $s_{N+1}^x = S(D; x)$, $k = k(\epsilon), \alpha_{(k)} = r$  yields
\begin{equation}
\Pr(s(D; x) \leq r) \leq 1-\epsilon + 1/(N+1) + 1/N.
\end{equation}

\end{proof}

\subsection{Proof of Theorem \ref{thm: dist_shift}}

\begin{proof}

From the proof of Theorem~\ref{thm: main_cp} (Eq.~\ref{eq: s_to_alpha}), 
\begin{equation}
\Pr[x_{N+1} \in C(\epsilon)] = \Pr[s_{N+1}^x \leq \alpha_{(k)}] \geq \Pr[s_{N+1}^x \leq s_{(k)}^x].
\end{equation}
The result follows from \cite{barber2023conformal} and using their Lemma 1.
\end{proof}

\subsection{Proof of Theorem \ref{thm: mix_shift}}

\begin{proof}

\begin{subequations}
\begin{gather}    
    d_{TV}(F, \Tilde{F}) 
    = \frac{1}{2} \int |F(x) - \Tilde{F}(x)| \ dx \\ 
    = \frac{1}{2} \int |\sum_{i=1}^m (w_i - \Tilde{w}_i) F_i(x)| \ dx \\ 
    \leq \frac{1}{2} \sum_{i=1}^m |w_i - \Tilde{w}_i| \int F_i(x) \ dx = \frac{1}{2} \sum_{i=1}^m |w_i - \Tilde{w}_i| 
    % = 1/2 ||w - \Tilde{w}||_1.
\end{gather}
\end{subequations}

\end{proof}

%% Use plainnat to work nicely with natbib. 
\bibliographystyle{ieeetr} % Format as IEEE Transactions
\bibliography{refs}

\section{Biography Section}
 
\vspace{11pt}

\begin{IEEEbiography}[{\includegraphics[width=1in,height=1.25in,clip,keepaspectratio]{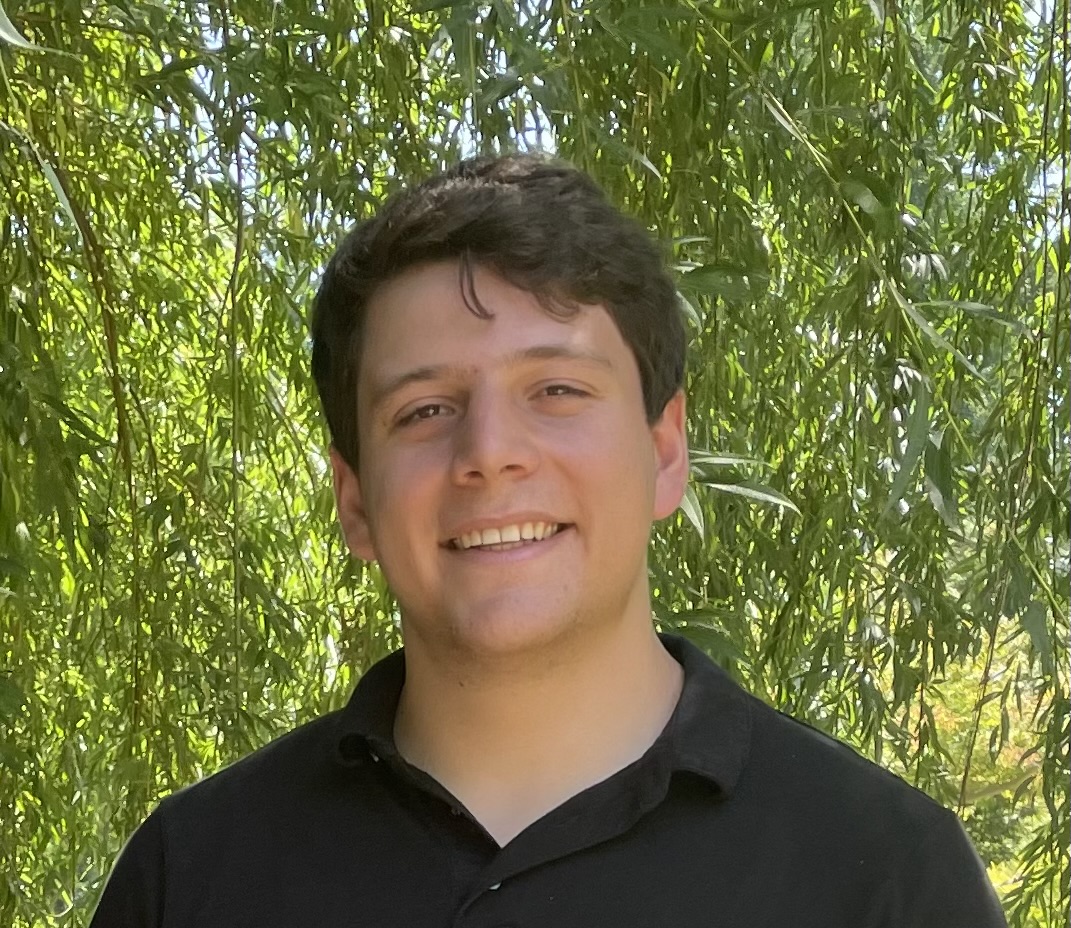}}]{Aaron O. Feldman}
received the B.S. degree in Information and Data Sciences from the California Institute of Technology in 2022 and is now pursuing a Ph.D. degree in Aeronautics and Astronautics at Stanford University. His current research focuses on statistical methods to guarantee safety and improve performance for robots operating under uncertainty.
\end{IEEEbiography}

\begin{IEEEbiography}
[{\includegraphics[width=1in,height=1.25in,clip,keepaspectratio]{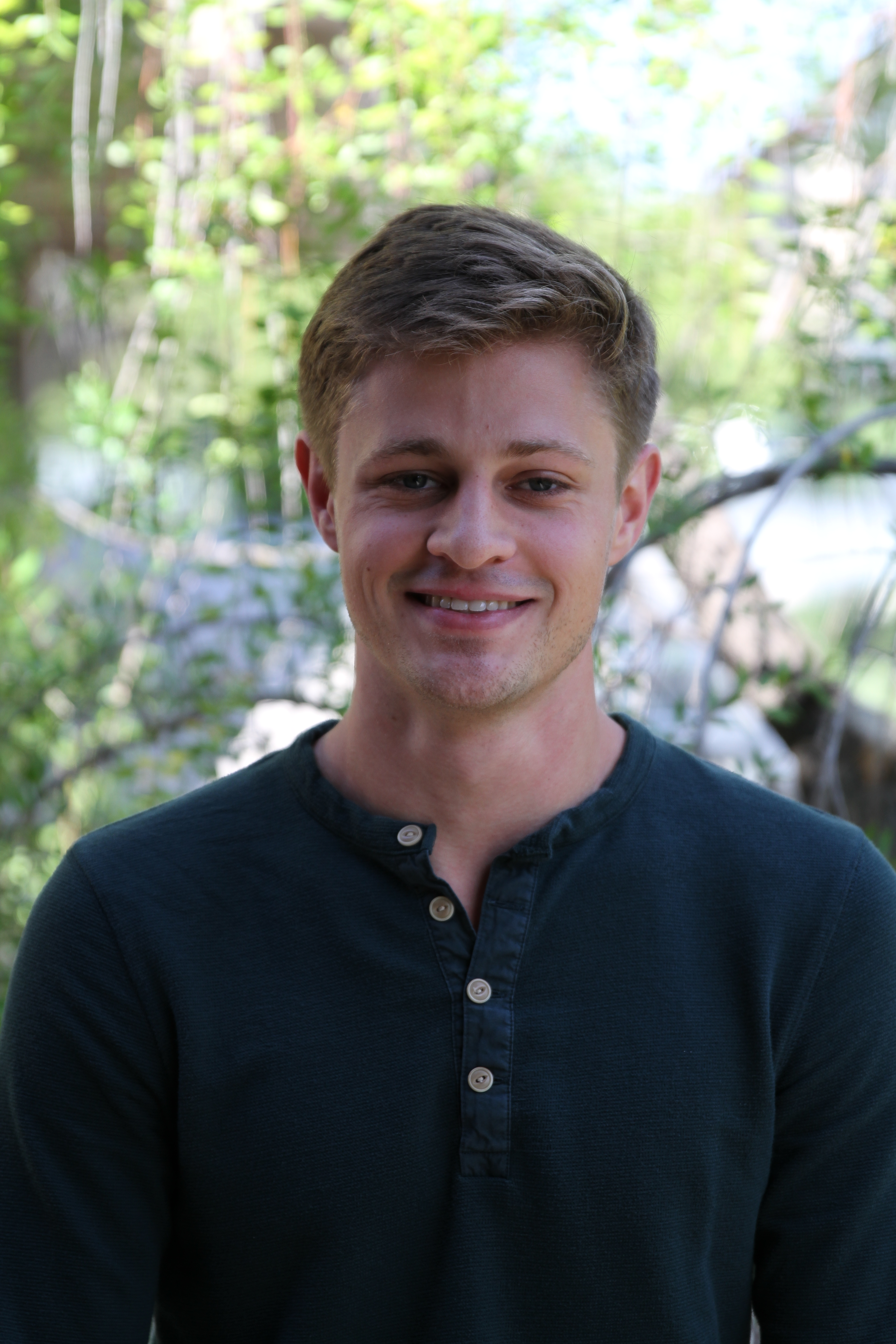}}]{Joseph A. Vincent}
received a B.S. degree in Aerospace Engineering from the University of Kansas in 2018, and M.S. and Ph.D. degrees in Aeronautics and Astronautics from Stanford University in 2020 and 2024, respectively. 
His research interests include evaluation of robotic systems using statistical and reachability-based methods.
\end{IEEEbiography}

\begin{IEEEbiography}
[{\includegraphics[width=1in,height=1.25in,clip,keepaspectratio]{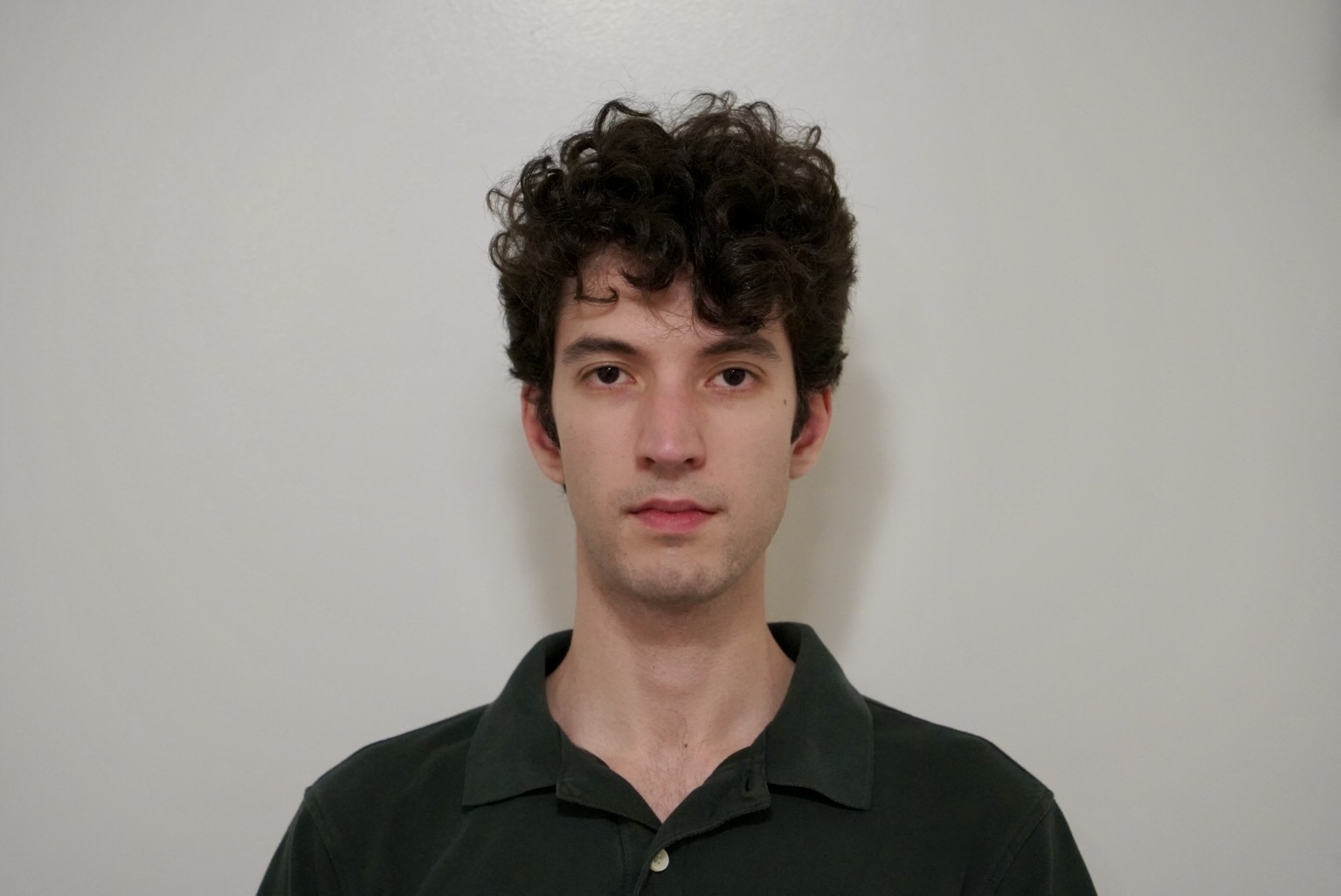}}]{Maximilian Adang}
received a B.S. degree in Mechanical Engineering from the California Institute of Technology in 2023 and is now a Ph.D. student in Aeronautics and Astronautics at Stanford University. His research investigates the application of end-to-end machine learning methods to guidance, navigation, and control systems used in UAVs and robotic manipulation on-orbit.
\end{IEEEbiography}

\begin{IEEEbiography}
[{\includegraphics[width=1in,height=1.25in,clip,keepaspectratio]{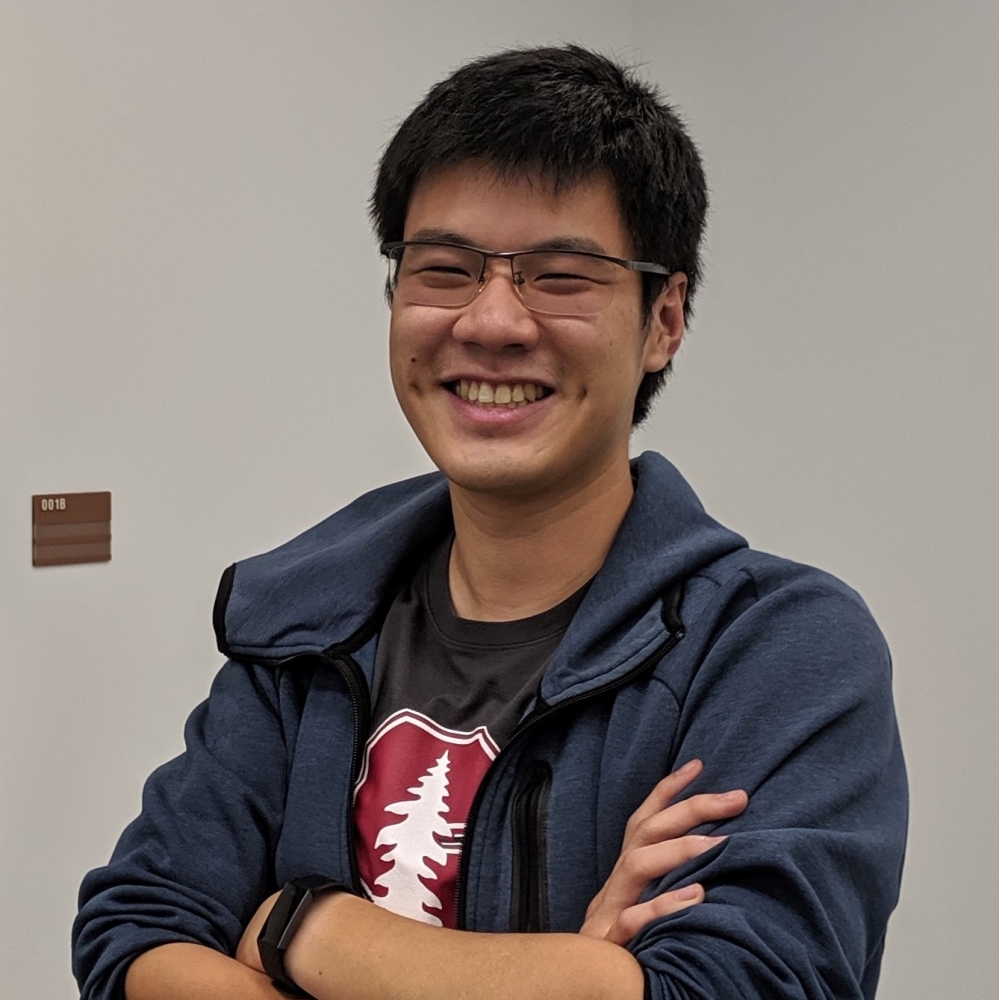}}]{JunEn Low}
is a graduate student in Mechanical Engineering at Stanford University. He received his B.Eng from the Singapore University of Technology and Design (SUTD) in 2015 where he worked on the dynamic modeling and control of unmanned aerial vehicles (UAVs). His current research interests are focused on zero-shot and end-to-end behavior cloning of visuomotor policies for UAVs.
\end{IEEEbiography}

\begin{IEEEbiography}[{\includegraphics[width=1in,height=1.25in,clip,keepaspectratio]{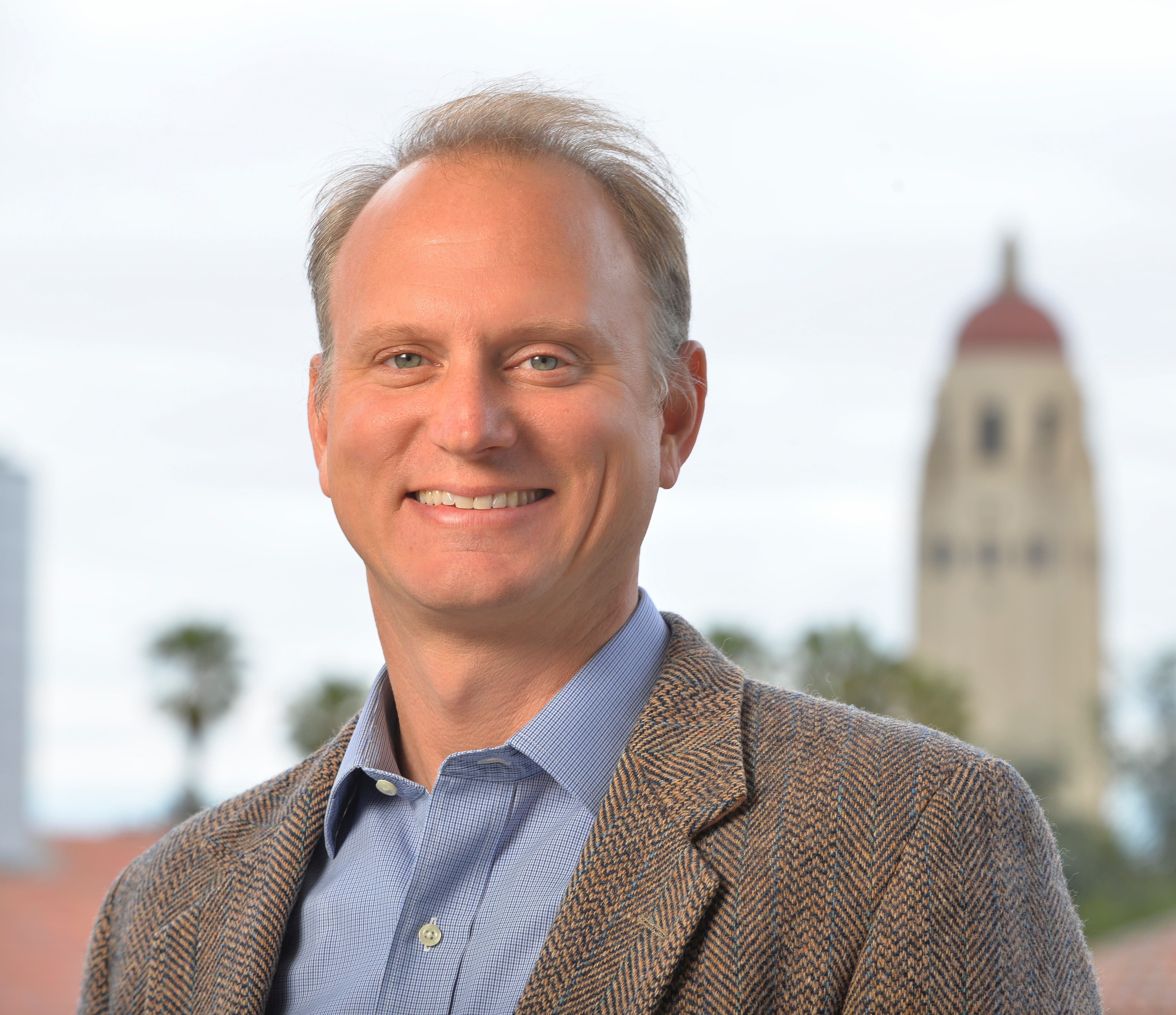}}]{Mac Schwager}
(Senior Member, IEEE) is Associate Professor of Aeronautics and Astronautics at Stanford University. He received a B.S. degree in Mechanical Engineering from Stanford University in 2000, and M.S. and Ph.D. degrees in Mechanical Engineering from the Massachusetts Institute of Technology in 2005 and 2009, respectively. His research interests include distributed algorithms for control, perception, and learning in groups of robots, and models of cooperation and competition in groups of engineered and natural agents. Dr. Schwager was the recipient of the NSF CAREER Award in 2014, the DARPA Young Faculty Award in 2018, and a Google Faculty Research Award
in 2018, and the IROS Toshio Fukuda Young Professional Award in 2019.
\end{IEEEbiography}



\vfill

\end{document}